\documentclass[12pt]{article}
\usepackage{amsmath}
\usepackage{graphicx}
\usepackage{natbib}
\usepackage{url}
\usepackage{xr}

\newcommand{\blind}{0}

\usepackage{enumerate}
\usepackage{booktabs}
\usepackage{appendix}

\usepackage{amsmath,amssymb,bbm}

\newcommand{\argmin}{\mathop{\arg\min}}
\newcommand{\bs}{\boldsymbol}
\newcommand{\diff}{\mathrm{d}}

\newcommand{\y}{\bs y}
\newcommand{\X}{\bs X}

\newcommand{\Q}{\bs Q_{\alpha}}
\newcommand{\bsbeta}{\bs \beta}
\newcommand{\bsepsilon}{\bs \varepsilon}

\newcommand{\tr}{\mathrm{tr}}
\newcommand{\FV}{\mathrm{FV}}
\newcommand{\JFV}{\mathrm{J}\text{-}\mathrm{FV}}
\newcommand{\LFV}{\mathrm{LFV}}
\newcommand{\CV}{\mathrm{CV}}

\newcommand{\normal}{N}

\usepackage[table]{xcolor}
\definecolor{Gray}{gray}{0.9}

\usepackage{amsthm}
\theoremstyle{definition}
\newtheorem{theorem}{Theorem}
\newtheorem*{theorem*}{Theorem}

\newtheorem{lemma}[theorem]{Lemma}
\newtheorem*{note*}{Note}

\newtheorem{prop}{Proposition}
\newtheorem{remark}{Remark}

\newcommand{\wrtX}{\bs X}

\begin{document}

\def\spacingset#1{\renewcommand{\baselinestretch}%
{#1}\small\normalsize} \spacingset{1}


\if0\blind
{
  \title{\bf A generalization gap estimation \\
  for overparameterized models \\ via the Langevin functional variance}
  \author{Akifumi Okuno\thanks{
    Corresponding author, okuno@ism.ac.jp; 
    A. Okuno was supported by JST CREST (JPMJCR21N3) and JSPS KAKENHI (21K17718, 22H05106). K. Yano was supported by JST CREST (JPMJCR1763), JSPS KAKENHI (19K20222, 21H05205, 21K12067), and MEXT (JPJ010217).}\hspace{.2cm}\\
    The Institute of Statistical Mathematics and RIKEN AIP\\
    and \\
    Keisuke Yano \\
    The Institute of Statistical Mathematics}
    \date{}
  \maketitle
} \fi

\if1\blind
{
  \bigskip
  \bigskip
  \bigskip
  \begin{center}
    {\LARGE\bf Title}
\end{center}
  \medskip
} \fi

\bigskip
\begin{abstract}
This paper discusses the estimation of the generalization gap, the difference between generalization performance and training performance, for overparameterized models including neural networks. 
We first show that a functional variance, a key concept in defining a widely-applicable information criterion, characterizes the generalization gap even in overparameterized settings where a conventional theory cannot be applied. 
As the computational cost of the functional variance is expensive for the overparameterized models, we propose an efficient approximation of the function variance,
the Langevin approximation of the functional variance~(Langevin FV). This method leverages only the $1$st-order gradient of the squared loss function, without referencing the $2$nd-order gradient; this ensures that the computation is efficient and the implementation is consistent with gradient-based optimization algorithms. 
We demonstrate the Langevin FV numerically by estimating the generalization gaps of overparameterized linear regression and non-linear neural network models, containing more than a thousand of parameters therein. 
\end{abstract}

\noindent%
{\it Keywords:}
Bayesian inference,
Langevin dynamics,
Neural networks,
Overparameterization.
\vfill

\newpage
\spacingset{1.5} 

\section{Introduction}

The great success of deep neural networks~\citep{lecun2015deep,Goodfellow-et-al-2016} has been reported in many applied fields, such as
natural language processing, image processing, and the natural sciences. This has now altered our understanding of \textit{generalization} in data science.
Generalization is our model's ability to adapt properly to new data drawn from the same distribution as that of the training data.
A classical discipline for generalization is 
\textit{the bias-variance dilemma}, that is:
the richness of models reduces the modeling bias but suffers from fitting spurious patterns, leading to poor generalization performance. 
However, the modern practice of deep neural networks yields counterexamples to this discipline, that is:
deep neural networks
are capable of exactly fitting the training data (known as interpolation) through the use of an overabundance of parameters of the models (known as over-parameterization) with accurately predicted test data~\citep{zhang2017,
belkin2018}. 
Several such surprising phenomena have been reported: 
the double descent phenomenon~\citep{hastie2022surprises,belkin2018,
geiger2020,
nakada2021asymptotic},
and
the multiple descent phenomenon~\citep{pmlr-v119-adlam20a,NEURIPS2020_1fd09c5f}.
These interesting but controversial phenomena give us a chance to rethink 
generalization in modern data science.

From a practical point of view,
estimating the generalization performance helps us understand what is actually happening.
A naive estimator of the generalization performance is the training performance.
However, real-case studies using deep neural networks~\citep{azulay2019why,zhang2021understanding} suggest that there is a gap between the generalization performance and the training performance, a \textit{generalization gap}.
Double descent phenomena also imply the existence of non-negligible generalization gap. These empirical evidences advocate the importance of accurately estimating the generalization gap.

Statistical tools for estimating the generalization gap have been reported in the literature on information criteria or optimism estimation.
The well-known Akaike information criterion (AIC; \citealp{akaike1974new}) estimates the generalization log-loss of the plug-in predictive density using a maximum likelihood estimator. 
The Takeuchi information criterion (TIC; \citealp{takeuchi1976distribution}) is a modification of AIC to deal with model misspecification.
The regularization information criterion (RIC; \citealp{shibata1989}) is an extension of TIC to accommodate maximum penalized likelihood estimation.
\cite{moody1992} and \cite{murata1994network} generalized RIC to an arbitrary loss and advocated its use in non-linear systems, such as neural networks. 
Mallows' $C_p$ (\citealp{mallows1973}) and Stein's unbiased risk estimates (SURE; \citealp{stein1981}) offer an elegant estimation scheme of the generalization gap using the covariance between a predictor and its outcome in Gaussian models with the $\ell_2$ loss. 
\cite{ramani2008} proposed an efficient Monte Carlo sampling-based method (Monte-Carlo SURE) to estimate the covariance.
Recently, new methods associated with Bayesian learning have been developed.
The deviance information criterion (DIC; \citealp{spiegelhalter2002}) and widely-applicable information criterion (WAIC; \citealp{watanabe2010asymptotic}) offer computationally-efficient devices that estimate the generalization loss of Bayesian learning. In particular, \cite{watanabe2010asymptotic} showed that
in a statistical model with a fixed dimension,
the generalization gap of Bayesian learning is 
asymptotically equal to
the functional variance (FV), that is,
the posterior variance of the log-likelihood.

However, studies on the data-driven measurement of the generalization gap in the overparameterized regime has been relatively scarce.
\cite{gao2016} extended Monte-Carlo SURE to analyze the generalization gap in deep neural networks.
\cite{thomas2020interplay} modified TIC for the aforementioned purpose.
These studies empirically investigated the gaps in common datasets, such as MNIST and CIFAR-10,
and
successfully developed useful tools to understand the generalization gap in the overparameterized regime.
However, these approaches lacked theoretical guarantees in overparameterized models although they provided guarantees in statistical models with a fixed dimension. 
This results in unclarified theoretical applicability of these approaches.  Furthermore, the computational costs of these approaches related to memory and speed are relatively high. This is because \cite{gao2016} requires training several times, and \cite{thomas2020interplay} needs the second-order gradient (Hessian) of the loss function.

In this paper, we focus on yet another tool to measure the generalization gap, the functional variance (FV) developed in \cite{watanabe2010asymptotic}. We present the following contributions:
\begin{itemize}
    \item \emph{Theoretical applicability:} 
    We prove that FV is asymptotically unbiased to the generalization gap of Bayesian learning for overparameterized linear regression models.

    \item \emph{Computational efficiency:}
    We propose a computationally-efficient approximation of FV, Langevin FV (LFV), by leveraging only the $1$st-order gradient of the loss function.

\end{itemize}

We show the theoretical applicability of FV as a measure of the generalization gap in overparameterized regime.
Our theory employs the overparameterized linear regression model (c.f.,~\citealp{hastie2022surprises,belkin2020two,bartlett2020benign}), but we should mention that this linear model can be regarded as a linear approximation of non-linear overparameterized models including deep neural networks and the linearization has been widely employed to analyze behaviours of deep neural networks (c.f.,~\citealp{NEURIPS2018_5a4be1fa,NEURIPS2019_dbc4d84b,
Chizatetal2019,yang2021icml}).
One of such theories is known as neural tangent kernel~\citep{NEURIPS2018_5a4be1fa}, that is, an approximation of neural networks in a small region of the parameters around an estimate or an initial value. 
The condition of the neural tangent kernel approximation has been investigated in \cite{Chizatetal2019}
and known to be applicable to any architecture of deep neural networks in \cite{yang2021icml}.

Although FV is theoretically favorable,
it is defined by using the full posterior covariance and 
the full Bayesian inference in the overparameterized regime is often prohibited because of the computational burden.
For efficiently computing the posterior, we employ the Langevin dynamics~(c.f., \citealp{Risken1996}), which sequentially adds a normal random perturbation to each update of the gradient descent optimization and obtains the stationary distribution approximating the posterior distribution~\citep{cheng2018sharp}.
This approach has several merits from the computational perspective. 
For example, the implementation is easy and consistent with the gradient-based optimization
that is the de-facto standard in deep neural network applications. 
Since the effect of linearization in theory remains unclear and controversial in practice (c.f.,~\citealp{leeetal2020}),
we confirm the applicability of the developed method to regression using neural networks with real datasets 
in Section~\ref{subsec:experiments_nn_real}.

The rest of the paper is organized as follows. 
Section~\ref{sec:theoretical_results_for_FV} provides our theoretical results on FV, Section~\ref{sec:langevin_approximation_of_bayesian_posterior} proposes Langevin FV, which is an efficient implementation of FV using Langevin dynamics, 
Section~\ref{sec:numerical_experiments} demonstrates the numerical experiments, and 
finally, Section~\ref{sec:conclusion} concludes this paper. 
Supplementary material gives proofs of all theorems.

\section{Theoretical results for the functional variance}
\label{sec:theoretical_results_for_FV}

In this section, we present theoretical results for the functional variance in overparameterized linear regression models. 
Throughout this paper, $\bs I_{n}$ denotes the $n\times n$ identity matrix for any $n \in \mathbb{N}$, and 
$\|\bs \beta\|_2=(\beta_1^2+\beta_2^2+\cdots+\beta_p^2)^{1/2}$ and $\|\bs \beta\|_{\infty}=\max_{i=1,2,\ldots,p}|\beta_i|$ for any vector $\bs \beta=(\beta_1,\beta_2,\ldots,\beta_p)^{\top} \in \mathbb{R}^p$.

\subsection{Problem setting}
\label{subsec:problem_setting}
We begin by introducing the problem setting for the theory.
We consider a linear regression model
\[
\y = \X \bsbeta_{0} + \bsepsilon,
\]
where $\y = (y_{1},\ldots,y_{n})^{\top} \in \mathbb{R}^{n}$ is a vector of observed outcomes, 
$\X=(\bs x_1,\bs x_2,\ldots,\bs x_n)^{\top} \in \mathbb{R}^{n \times p}$ is an $n\times p$ random design matrix, 
$\bsbeta_{0}\in\mathbb{R}^{p}$ is an unknown coefficient vector,
and $\bsepsilon=(\varepsilon_{1},\ldots,\varepsilon_{n})^{\top} \in \mathbb{R}^n$ is a vector of independently and identically distributed (i.i.d.)~error terms with mean zero and variance $0<\sigma_{0}^{2}<\infty$. 
Our interest is the overparameterized situation, where the number of regressors $p$ is larger than the sample size $n$:
\[
n \le p.
\]

We take the quasi-Bayesian approach on the vector $\bs \beta$. Working under the quasi-likelihood $f(y_i \mid \bs x_{i},\bsbeta)$ with a Gaussian distribution with mean $\bs x_{i}^{\top}\bsbeta$ and variance $\sigma^{2}_{0}$, and a Gaussian prior $\normal_{p}(\bs 0, \{\sigma^{2}_{0}/(\alpha n)\} \bs I_{p})$ on $\bs \varepsilon$ with mean $\bs 0$ and covariance matrix $\{\sigma^{2}_{0}/(\alpha n)\} \bs I_{p}$, we obtain the quasi-posterior distribution of $\bs \beta$:
\begin{align}
    \Pi_{\alpha}(\diff \bsbeta \mid \y,\X)
    =
    (2\pi )^{-p/2}(\mathrm{det}\Q)^{-1/2}
    \mathrm{e}^{-(\bsbeta-\hat{\bsbeta}_{\alpha})^{\top}\Q^{-1}(\bsbeta-\hat{\bsbeta}_{\alpha})/2} \diff \bsbeta,
    \label{eq:posterior}
\end{align}
where $\hat{\bsbeta}_{\alpha}$ is the maximum a posterior estimate
\begin{align}
\hat{\bsbeta}_{\alpha}
:=
\left(n^{-1}\X^{\top}\X + \alpha \bs I_{p}\right)^{-1}
n^{-1}\X^{\top}\y
=
\argmin_{\bs \beta \in \mathbb{R}^d}
    \left\{
        \ell_{\alpha}(\bs \beta)
    \, := \, 
        n^{-1}\|\bs y-\bs X\bs \beta\|_2^2
        +
        \alpha \|\bs \beta\|_2^2
    \right\}
    \label{eq:ridge}
\end{align}
and $\Q$ is the matrix
\begin{align}
\Q:=
n^{-1}\sigma^{2}_{0}
\left( n^{-1}\X^{\top}\X+\alpha\bs I_{p} \right)^{-1}.
\label{eq:Q}
\end{align}

We now analyze the Gibbs generalization gap
\begin{align}
    \Delta(\alpha ; \wrtX)
    :=
    \frac{1}{2\sigma_0^2}
    \bigg\{
    \mathbb{E}_{\bs y,\bs y^*}\left[
    \mathbb{E}_{\bs \beta}\left[
        \|\bs y^*-\bs X\bs \beta\|_2^2
    \right]
    \right]
    -
    \mathbb{E}_{\bs y}\left[
    \mathbb{E}_{\bs \beta}\left[
        \|\bs y-\bs X\bs \beta\|_2^2
    \right]
    \right]
    \bigg\},
    \label{eq:generalization_error}
\end{align}
where 
$\y^{*}$ is an independent copy of $\y$ with given $\X$,
$\mathbb{E}_{\bs \beta}[\cdot]$ is 
the expectation with respect to the quasi-posterior distribution (\ref{eq:posterior}),
and 
$\mathbb{E}_{\bs y}[\cdot]$ (and $\mathbb{E}_{\bs y, \bs y^{*}}[\cdot]$) is
the expectation with respect to $\bs y$ (and $\bs y, \bs y^{*}$, respectively,) with given $\bs X$.
The Gibbs generalization gap considers 
the generalization gap for
one stochastic sample from the quasi-posterior distribution as  estimates of the parameters. 
To estimate the Gibbs generalization gap from the current observations,
we focus on the functional variance
\begin{align}
    \FV(\alpha ; \wrtX)
    :=
    \sum_{i=1}^{n} \mathbb{V}_{\bs \beta}[\log f(y_{
i}\mid \bs x_{i},\bs \beta)],
    \label{eq:FV}
\end{align}
where $\mathbb{V}_{\bs \beta}$ is 
the variance with respect to the quasi-posterior distribution (\ref{eq:posterior}).

\subsection{Asymptotic unbiasedness of the functional variance}

Here we present our theoretical findings on the functional variance in the overparameterized regime.
The following theorem
shows that the functional variance $\FV(\alpha ; \wrtX)$ is an asymptotically unbiased estimator of the Gibbs generalization gap $\Delta(\alpha ; \wrtX)$ for the overparameterized linear regression with Gaussian covariates. 
Its proof is given in Supplement~\ref{sec:mainproof}.

\begin{theorem}
\label{theo:convergence_of_FV_Gaussian}
    Let $n,p \in \mathbb{N}$, and let $\bs \Sigma$ be a $p \times p$ non-zero and non-negative definite matrix. 
    Let $\bs x_1,\bs x_2,\ldots,\bs x_n \in \mathbb{R}^{p}$ be i.i.d.~random vectors from the Gaussian distribution with a mean zero and a covariance matrix $\bs \Sigma$. 
    There exists an absolute constant $C>0$, such that we have, for any $\varepsilon,\alpha>0$,
    \[
    \mathbb{P}_{\wrtX}\left( 
            |\mathbb{E}_{\bs y}[\FV(\alpha ; \wrtX)]-\Delta(\alpha ; \wrtX)| > \varepsilon
        \right)
        \le 
    C
    \left(
        \frac{\xi^{2}}{\alpha^2}\frac{1}{n\varepsilon}
        +
        \frac{\xi^{4}}{\alpha^4}\frac{1}{n\varepsilon}
        +
        \frac{\xi^{4}b^4}{ \alpha^2 \sigma_0^{4}}
        \frac{1}{n\varepsilon^{2}}
        +
        \frac{1}{n}
    \right),
    \]
    where $\xi := \tr\{\bs \Sigma\}$, $b := p^{1/2}\|\bs \beta_0\|_{\infty}$, and
    $\mathbb{P}_{\bs X}$ denotes the probability with respect to $\bs X$.
    Furthermore, 
    under the conditions (C1) $\sup_{p\in\mathbb{N}}\xi<\infty$ and (C2) $\sup_{p\in\mathbb{N}}b<\infty$,
    we have 
    \[
    \mathbb{P}_{\wrtX}(|\mathbb{E}_{\bs y}[\FV(\alpha ; \wrtX)]-\Delta(\alpha ; \wrtX)| \le C'a_n/\sqrt{n}) \to 1
    \]
    for some $C'>0$ not depending on $n$ and $p$
    with an arbitrary slowly increasing sequence $a_n$.
\end{theorem}

Theorem \ref{theo:convergence_of_FV_Gaussian} implies that FV successfully estimates the Gibbs generalization gap in the overparameterized regime, which supports the use of FV, even in the overparameterized regime.
Interestingly, its rate of convergence $(a_{n}/\sqrt{n})$ is not affected by the dimension $p$ but affected by the value of $\sup_{p\in\mathbb{N}}\xi$. As $\sup_{p\in\mathbb{N}}\xi$ gets larger, the decay of the difference becomes slower.
Furthermore, the result does not restrict the true distribution of the additive error to the Gaussian distribution; in contrast, some classical theories such as SURE require the error term to either be Gaussian or be of related distributions.

Let us mention the conditions in the theorem.
Condition (C1) indicates the trace boundedness of $\bs X^{\top}\bs X$ (divided by $n$) because $\bs \Sigma=\mathbb{E}[\bs X^\top \bs X]/n$. 
Both shallow and deep neural network models satisfy this condition in general settings~(see, e.g., \citealp{karakida2019universal}). 
Condition (C2) together with Condition (C1) indicates that $n^{-1}\|\bs \mu\|_2^2 \le n^{-1}\|\bs X\|_{\text{F}}^2 \|\bs \beta_0\|_2^2 
$ is upper-bounded by some constant $C>0$ with probability approaching $1$.
It implies that the average of the entries in the outcome expectation $\bs \mu$ is upper-bounded with high probability.

The main ingredient of the proof of Theorem \ref{theo:convergence_of_FV_Gaussian} is the explicit identity of the difference between the generalization gap and FV. 
See Lemma~\ref{lem:explicit_difference} below.
Its proof is presented in Supplement~\ref{supp:proofoflemma}.
For two matrices $\bs A=(a_{ij}),\bs B=(b_{ij})$ of the same size,
$\bs A \circ \bs B=(a_{ij}b_{ij})$ denotes the Hadamard element-wise product, and let 
$\bs A^{\otimes 2}:=\bs A\bs A^{\top}$.

\begin{lemma}\label{lem:explicit_difference}
We have
\begin{align}
\Delta(\alpha ; \wrtX)&=\tr\{\bs H_{\alpha}\}\quad \text{and}
 \label{eq:hatmatrix}\\
\mathbb{E}_{\bs y}[\FV(\alpha ; \wrtX)]-\Delta(\alpha ; \wrtX)
&=
-\dfrac{3}{2}\tr\{\bs H_{\alpha} \circ \bs H_{\alpha}\}
+\tr\{\bs H_{\alpha} \circ \bs H_{\alpha}^2\} \nonumber \\
&\hspace{5em}
+\dfrac{1}{\sigma_0^2}\tr\{
    \bs H_{\alpha} \circ ((\bs I_n-\bs H_{\alpha})(\bs X\bs \beta_0))^{\otimes 2}\},
    \label{eq:explicit_difference}
\end{align}
where $\bs H_{\alpha}$ is the regularized hat matrix
\begin{align}
    \bs H_{\alpha} := n^{-1}\bs X\left( n^{-1}\bs X^{\top}\bs X+\alpha \bs I_p\right)^{-1}\bs X^{\top}.
    \label{eq:quasi_hat}
\end{align}
\end{lemma}

Lemma \ref{lem:explicit_difference} highlights the role of the regularized hat matrix $\bs H_{\alpha}$ in evaluating the generalization gap and the residual.
Same as the fixed-dimension theory, the hat matrix controls the magnitude of the generalization gap.
With singular values $s_1 \ge s_2 \ge \cdots \ge s_n \ge 0$ of the matrix $\X$, under the conditions in Theorem \ref{theo:convergence_of_FV_Gaussian}, we have
\[ 
\Delta(\alpha ; \wrtX)
=
\tr\{\bs H_{\alpha}\}
=
\sum_{i=1}^{n} \frac{s_{i}^2/n}{(s_{i}^2/n) + \alpha}
\to 
\exists \Delta_{\infty}(\alpha) \in [0,\infty) \quad (n \to \infty)
\]
as $\tr\{\bs H_{\alpha}\} \le \alpha^{-1}n^{-1}(\sum_{i=1}^{n}s_{i}^2)<\infty$ with the probability approaching $1$ as $n \to \infty$. 
Consider a simple case where $\bs X$ has a fixed intrinsic dimension $p_* \ll n$ with growing $n,p$, namely $s_{1} \ge s_{2} \ge \cdots \ge s_{p_*}>0=s_{p_*+1}=s_{p_*+2}=\cdots=s_{n}$.
Then, FV approaches the fixed intrinsic dimension $p_*$ (as $\alpha \to 0$), which coincides with the degrees of freedom~\citep{mallows1973,Ye1998} and AIC~\citep{akaike1974new} for linear regression equipped with $p_*$ regressors.

\begin{remark}[Sample-wise and joint log-likelihoods]
One may think the use of the posterior covariance of the joint log-likelihood
\begin{align*}
    \JFV(\alpha;\wrtX)
    :=
    \mathbb{V}_{\bs \beta}[
        \log f(\bs y \mid \bs X, \bs \beta)
    ]
\end{align*}
instead of FV~(\ref{eq:FV}) using sample-wise log-likelihood.
In this case, we have the following expression of the difference between the J-FV and the generalization gap.
Its proof is in Supplement~\ref{supp:proof_of_lem:expectation_of_J-FV}. 
\begin{prop}
\label{lem:expectation_of_J-FV}
For any $\alpha>0$, we have
\begin{align*}
        \mathbb{E}_{\bs y}[\JFV(\alpha;\wrtX)]
        -
        \Delta(\alpha ; \wrtX)
        &=
        -
        \frac{3}{2}\tr\{\bs H_{\alpha}^2\}
        +
        \tr\{\bs H_{\alpha}^3\}
        +
        \frac{1}{\sigma_0^2}
        \tr\{\bs H_{\alpha}((\bs I-\bs H_{\alpha})(\bs X\bs \beta_0))^{\otimes 2}\}.
\end{align*}
\end{prop}

This exhibits an interesting correspondence between FV and J-FV (see Lemma~\ref{lem:explicit_difference} and Proposition~\ref{lem:expectation_of_J-FV}); replacing the Hadamard products in $\mathbb{E}_{\bs y}[\FV(\alpha ; \wrtX)]$ with simple matrix products yields $\mathbb{E}_{\bs y}[\JFV(\alpha;\wrtX)]$. 
Since the simple matrix products in $\mathbb{E}_{\bs y}[\JFV(\alpha;\wrtX)]$ do not vanish as $n \to \infty$, J-FV is not an asymptotically unbiased estimator of the generalization gap. 

\end{remark}

\begin{remark}[Extension of the theorem]
Theorem~\ref{theo:convergence_of_FV_Gaussian} is further extended in Theorem~\ref{theo:convergence_of_FV}, which proves the convergence of FV under milder conditions. 
Its proof is shown in Supplement~\ref{sec:mainproof}.

\begin{theorem}
\label{theo:convergence_of_FV}
Let $n,p \in \mathbb{N}$ and assume the setting in Section~\ref{subsec:problem_setting}. 
Let $\bs X\in \mathbb{R}^{n \times p}$ be a random matrix and let $\mathbb{S}^{n-1}=\{\bs u \in \mathbb{R}^d \mid \|\bs u\|_2=1\}$ be a $(n-1)$-sphere. 
Let $q_{1},q_{2},\ldots,q_{n}$ be marginal probability densities of the left-singular vectors $\bs u_{1},\bs u_{2},\ldots,\bs u_{n} \in \mathbb{S}^{n-1}$ of $\bs X$. 
Write
\[
b := p^{1/2}\|\bs \beta_0\|_{\infty}, \quad
\psi := \max_{i=1,2,\ldots,n} 
\max_{\bs u \in \mathbb{S}^{n-1}}
\{q_{i}(\bs u)\} \int_{\mathbb{S}^{n-1}}\diff \bs v, 
\quad
\eta := \frac{1}{n} \tr\{\bs X^{\top}\bs X\}.
\]
Then, there exists an absolute constant $C>0$ such that, for any $\varepsilon,\alpha>0$, we have 
\begin{align}
    \mathbb{P}_{\wrtX}\left(
        |\mathbb{E}_{\bs y}[\FV(\alpha ; \wrtX)] - \Delta(\alpha ; \wrtX)|
        >
        \varepsilon 
        \mid \eta
    \right)
&\le 
    C \left(
        \frac{\psi \eta^2}{\alpha^2}
        \frac{1}{n\varepsilon}
        +
        \frac{\psi \eta^4}{\alpha^4}
        \frac{1}{n\varepsilon}
        +
        \frac{\psi \eta^4 b^4}{\alpha^2 \sigma_0^4}
        \frac{1}{n\varepsilon^{2}}
    \right),
    \label{eq:bound_bias_general}
\end{align}
where
$\mathbb{P}_{\wrtX}(\cdot \mid \eta)$ denotes the conditional probability of $\wrtX$ given $\eta$.
\end{theorem}

Theorem \ref{theo:convergence_of_FV_Gaussian} (Gaussian covariates) follows from Theorem~\ref{theo:convergence_of_FV} with $\psi=1$; see Proposition~7.1 of \cite{eaton1989group}, which proves that the left-singular vectors of the Gaussian design matrix follow a uniform distribution over the unit $(n-1)$-sphere $\mathbb{S}^{n-1}$, i.e., $q_{i}(\bs u)=\mathbbm{1}_{\mathbb{S}^{n-1}}(\bs u)/\int_{\mathbb{S}^{n-1}}\diff \bs v$ with $\mathbbm{1}_{\mathbb{S}^{n-1}}(\cdot)$ the indicator function of $\mathbb{S}^{n-1}$.
\end{remark}

\begin{remark}[Definition of the generalization gap]
Regarding the definition of the generalization gap $\Delta(\alpha ; \wrtX)$, we may consider another design matrix $\bs X^*$ for generating $\bs y^*$. 
However, the covariate difference is less effective in the generalization gap~(\ref{eq:generalization_error}) since we focus on the prediction of the conditional random variable $\bs y \mid \bs X$ and not the covariate $\bs X$. 
For theoretical simplicity, we employed a single design matrix $\bs X$ to generate both outcomes $\bs y,\bs y^*$. 
\end{remark}

\section{Langevin Functional Variance}
\label{sec:langevin_approximation_of_bayesian_posterior}

Despite FV being theoretically attractive as proved in Theorem~\ref{theo:convergence_of_FV}, there exist two computational difficulties: 
\begin{enumerate}[{(D1)}]
\item generating samples from the quasi-posterior~(\ref{eq:posterior}) requires the $p \times p$ matrix $\bs Q_{\alpha}$, which is computationally intensive for overparameterized models $n \ll p$, and 
\item computing the quasi-posterior~(\ref{eq:posterior}) is inconsistent with gradient-based optimization approaches, such as the stochastic gradient descent~(SGD), which are often used in optimizing overparameterized models.
\end{enumerate}
To resolve these difficulties (D1) and (D2), 
we consider a Langevin approximation of the quasi-posterior in Section~\ref{subsec:Langevin_posterior}, and propose Langevin FV~(LFV) for linear models in Section~\ref{subsec:Langevin_FV}; we then extend it to non-linear models in Section~\ref{subsec:non-linear}. 
We compare FV and the proposed LFV with the existing estimators based on TIC~\citep{takeuchi1976distribution} in Section~\ref{subsec:TIC}.

\subsection{Langevin approximation of the quasi-posterior}
\label{subsec:Langevin_posterior}

A key idea of our algorithm is approximating the quasi-posterior~(\ref{eq:posterior}) via a Langevin process. Starting with an estimator $\bs \gamma^{(1)}:=\hat{\bs \beta}_{\alpha}$, we stochastically update $\bs \gamma^{(t)}$ by
\begin{align}
    \bs \gamma^{(t+1)}
    =
    \bs \gamma^{(t)}
    -
    \frac{1}{4} \delta \kappa_n \frac{\partial \ell_{\alpha}(\bs \gamma^{(t)})}{\partial \bs \gamma}
    +
    \delta^{1/2} \bs e^{(t)}
    \quad (t=1,2,\cdots),
    \label{eq:Langevin}
\end{align}
where $\kappa_n:=n/\sigma_0^2$ and $\delta>0$ are  user-specified parameters, let $\{\bs e^{(1)},\bs e^{(2)},\cdots \}$ denote a sequence of i.i.d.~standard normal random vectors, and $\ell_{\alpha}$ is defined in (\ref{eq:ridge}). 
Then, the distribution of the Langevin process~\eqref{eq:Langevin} approximates the quasi-posterior as discussed in the following. 
First, the Langevin process~(\ref{eq:Langevin}) is a discretization of the Ornstein-Uhlenbeck process
\begin{align*}
    \diff \tilde{\bs \gamma}_{\tau}
    =
    -\frac{1}{2}\bs Q_{\alpha}^{-1}(\tilde{\bs \gamma}_{\tau}-\hat{\bs \beta}_{\alpha}) \diff \tau 
    +
    \diff \bs e_{\tau}
\end{align*}
equipped with a Wiener process $\{\bs e_{\tau}\}_{\tau \ge 0}$, i.e., $\bs e_{\tau}-\bs e_{\tau'} \sim \normal_p(\bs 0,(\tau-\tau')\bs I_p)$ for any $\tau>\tau' \ge 0$. The distribution of $\tilde{\bs \gamma}_{\tau}$ with the initialization $\tilde{\bs \gamma}_{0}:=\hat{\bs \beta}_{\alpha}$ coincides with the quasi-posterior
\begin{align}
    \mathcal{D}(\tilde{\bs \gamma}_{\tau})
    =
    N_p(\hat{\bs \beta}_{\alpha},\bs Q_{\alpha}) \quad (\tau>0);
\label{eq:OU_convergence}
\end{align}
See \citet{Risken1996} page 156 for an example of the distribution~(\ref{eq:OU_convergence}). 
Next, with $n$ and $p$ fixed, for any sufficiently small $\phi>0$, 
Theorem 2 of \citet{cheng2018sharp} evaluates the 
$1$-Wasserstein distance between the distributions of the Ornstein-Uhlenbeck and the Langevin processes as
\begin{align}
    W_1(\mathcal{D}(\tilde{\bs \gamma}_{t\delta}),\mathcal{D}(\bs \gamma^{(t)})) 
    \le 
    \phi
    \label{eq:W1_Langevin}
\end{align}
for some $\delta =O(\phi^2/p)$ and $t \asymp p/\phi^2$; 
see Proposition~\ref{prop:W1_evaluation} in Supplement~\ref{sec:supporting_propositions} for details. 
Thus, the relations (\ref{eq:OU_convergence}) and (\ref{eq:W1_Langevin}) imply that the distribution $\mathcal{D}(\bs \gamma^{(t)})$ approximates the quasi-posterior $\normal_p(\hat{\bs \beta}_{\alpha},\bs Q_{\alpha})$.

This Langevin approximation resolves difficulties (D1) and (D2) since 
\begin{enumerate}[{(S1)}]
\item the Langevin process~(\ref{eq:Langevin}) is computed with only the $1$st order gradient of the squared loss function $\ell_{\alpha}(\bs \gamma)$ defined in (\ref{eq:ridge}), meaning that the large matrix $\bs Q_{\alpha}$ is not explicitly computed, and
\item the Langevin process~(\ref{eq:Langevin}) is in the form of the gradient descent up to the normal noise $\delta^{1/2}\bs e^{(t)}$; this process can be implemented consistently with gradient-based algorithms, which are often used to optimize overparameterized models. 
\end{enumerate}

\begin{remark}[The other approach for the posterior approximation]
Besides the Langevin dynamics, we can use the other approaches for approximating the posterior distribution. One such approach is to utilize the stochastic noise in stochastic gradient descent~\citep{pmlr-v32-satoa14,mandt2017sgd}. Though powerful, this approach relies on the normal approximation in the stochastic gradient and the normal approximation is hardly expected in the overparameterized regime, which results in the computational instability.
\end{remark}

\subsection{Langevin functional variance}
\label{subsec:Langevin_FV}

Using the Langevin approximation~(\ref{eq:Langevin}), we propose LFV:
for a time step $T \in \mathbb{N}$,
let $\{\bs \gamma^{(t)}\}_{t=1,2,\ldots,T}$ be the samples
from the Langevin process.
Let $\hat{\mathbb{V}}_{\bs \gamma}[\cdot]$ denotes an empirical variance with respect to $\{\bs \gamma^{(t)}\}_{t=1,2,\ldots,T}$.
Let $\mu_i^{(t)}:=\bs x_i^{\top}\bs \gamma^{(t)}$. 
Then, we define LFV as
\begin{align}
    \LFV(\alpha;\wrtX)
    &:=
    \sum_{i=1}^{n}
    \hat{\mathbb{V}}_{\bs \gamma}[\log f(y_i \mid \bs x_i,\bs \gamma)] \nonumber \\
    &=
    \scalebox{0.9}{$\displaystyle 
    \sum_{i=1}^{n}
    \frac{1}{T}
    \sum_{t=1}^{T}\left\{
     \frac{1}{2\sigma_0^2}
     \left(y_i-\mu_i^{(t)}\right)^2
     -
     \frac{1}{T}\sum_{t'=1}^{T} \frac{1}{2\sigma_0^2}\left(y_i-\mu_i^{(t')}\right)^2
    \right\}^2
    $}.
    \label{eq:Langevin_FV}
\end{align}
As shown in \eqref{eq:W1_Langevin}, the distribution of the Langevin process approaches the quasi-posterior: LFV is expected to approximate FV for a sufficiently large $T \in \mathbb{N}$. 
Thus LFV is expected to inherit the nature of FV, such as the asymptotic unbiasedness when estimating the generalization gap.

The distributional approximation of the high-dimensional vector $\bs \gamma \in \mathbb{R}^p$ faces the curse of dimensionality. However, FV is a statistic taking a value in $\mathbb{R}$; empirically, FV and LFV approximate the generalization gap $\Delta(\alpha ; \wrtX)$ with a reasonable number of samples $T \in \mathbb{N}$, even if the dimension $p$ is relatively large. See numerical experiments in Section~\ref{subsec:experiments_ics_convergence}. 

Lastly, let us mention the computational time of FV and LFV.
For each iteration, the number of operations to obtain the gradient in LFV is $O(p^{2})$,
while that to compute the covariance matrix in FV is $O(p^3)$: therefore, the computational time is expected to be reduced, as also demonstrated for linear models in Remark~\ref{remark:computational_complexity}.

\subsection{Application to non-linear models}
\label{subsec:non-linear}
This subsection provides a procedure in applying LFV~(\ref{eq:Langevin_FV}) to a non-linear overparameterized model $g_{\bs \theta}(\bs z_i)$ equipped with a parameter vector $\bs \theta$ and an input vector $\bs z_i$. 
The procedure is as follows: first, by replacing $\ell_{\alpha}(\bs \beta)$ in the Langevin process (\ref{eq:Langevin}) with 
\begin{align*}
    \rho_{\alpha}(\bs \theta)
    &:=
    \frac{1}{n}\sum_{i=1}^{n}(y_i-g_{\bs \theta}(\bs z_i))^2
    +
    \alpha \|\bs \theta\|_2^2,
\end{align*}
the update (\ref{eq:Langevin}) yields a sequence $\{\bs \gamma^{(t)}\}_{t=1,2,\ldots}$.
Next, by substituting $\mu_i^{(t)}:=g_{\bs \gamma^{(t)}}(\bs z_i)$ in (\ref{eq:Langevin_FV}),
we obtain LFV for the non-linear model $g_{\bs \theta}$. 
We here emphasize that, 
while the exact form of the full quasi-posterior for non-linear models 
 is difficult to obtain
 and thus computing FV for non-linear models is almost prohibited, LFV can be applied to even such non-linear models.

The gradient of the overparameterized model $g_{\bs \theta}$ is compatible with the gradient of its linear approximation. 
This application works if the linear approximation of non-linear neural network is reasonable around the initial estimate. Again, we remark that this assumption is important in the neural tangent kernel literature~\citep{NEURIPS2018_5a4be1fa,NEURIPS2019_dbc4d84b}.

\subsection{Comparison to the existing methods}
\label{subsec:TIC}

Here, we compare FV and LFV to existing methods for the estimation of the generalization gap using TIC~\citep{takeuchi1976distribution}.
TIC estimates the generalization gap of regular statistical models by using a quantity
\begin{align}
    \tr\left\{\hat{\bs F}\hat{\bs G}^{-1}\right\}
    \label{eq:TIC}
\end{align}
equipped with $p \times p$ matrices 
\begin{align*} 
\hat{\bs F}
&:=
\sum_{i=1}^{n}
\frac{\partial \log f(y_i \mid \bs x_i,\bs \beta)}{\partial \bs \beta}
\frac{\partial \log f(y_i \mid \bs x_i,\bs \beta)}{\partial \bs \beta^{\top}}
\bigg|_{\bs \beta=\hat{\bs \beta}}
\text{ and } 
\hat{\bs G}
:=
\sum_{i=1}^{n}
\frac{\partial^2 \log f(y_i \mid \bs x_i, \bs \beta)}{\partial \bs \beta \partial \bs \beta^{\top}}
\bigg|_{\bs \beta=\hat{\bs \beta}},
\end{align*}
where $\hat{\bs \beta}$ is a maximum likelihood estimate.
However, TIC cannot be applied to singular statistical models whose Hessian matrix $\hat{\bs G}$ degenerates (i.e., the inverse of $\hat{\bs G}$ does not exist), such as overparameterized models. 
To overcome this limitation of TIC, \citet{thomas2020interplay} replace the inverse matrix $\hat{\bs G}^{-1}$ in (\ref{eq:TIC}) with the $\kappa$-generalized inverse matrix $\hat{\bs G}^+_{\kappa}$; $\hat{\bs G}^+_{\kappa}$ is defined as $\bs V \bs \Sigma^{+}_{\kappa} \bs V^{\top}$, 
where $\bs V$ is a matrix of the eigenvectors of $\hat{\bs G}$, and $\bs \Sigma^{+}_{\kappa}$ is the diagonal matrix of which the $j$-th  diagonal component is
\begin{align*}
    \lambda_j (\hat{\bs G}_{\kappa}^+)
    :=
    \begin{cases}
        1/\lambda_j(\hat{\bs G}) & (\lambda_j(\hat{\bs G}) > \kappa) \\
        0 & (\lambda_j(\hat{\bs G}) \le \kappa) \\
    \end{cases},
    \quad
    (j=1,2,\ldots,p)
\end{align*}
with $\{\lambda_j (\hat{\bs G}): j=1,\ldots,p\}$ singular values of $\hat{\bs G}$. 
This TIC modification of \citet{thomas2020interplay} is numerically examined in Section~\ref{subsec:experiments_ics_convergence}, for overparameterized linear regression; it rather estimates the number of nonzero eigenvalues in $\hat{\bs G}=n^{-1}\bs X^{\top}\bs X$, i.e., the number of nonzero singular values of $\bs X$, but not the generalization gap $\Delta(\alpha ; \wrtX)$.

Instead of modifying TIC, we can employ RIC proposed by  \citet{shibata1976selection}. 
RIC replaces the Hessian matrix $\hat{\bs G}$ in TIC with the regularized inverse matrix $(\hat{\bs G}+\alpha \bs I)^{-1}$. \citet{moody1992} and \citet{murata1994network} further generalized RIC to arbitrary loss functions and demonstrated its application in shallow neural network models with a small number of hidden units. 
RIC has much in common with FV; specifically, RIC is also an asymptotically unbiased estimator of the Gibbs generalization gap $\Delta(\alpha ; \wrtX)$ for overparameterized linear regression, and numerically, RIC behaves similarly to FV (and LFV) for the experiments discussed in Section~\ref{subsec:experiments_ics_convergence}. 
However, RIC still requires the inverse matrix of the $p \times p$ matrix $\hat{\bs G}+\alpha\bs I$, and its computational cost is intensive in the overparameterized setting ($p \ge n$).

\section{Numerical Experiments}
\label{sec:numerical_experiments}

This section 
presents the numerical evaluation of the Langevin FV~(\ref{eq:Langevin_FV}) using  both linear and non-linear models.

\subsection{LFV and baselines in linear models}
\label{subsec:experiments_ics_convergence}

We first evaluate LFV~(\ref{eq:Langevin_FV}) and compare it with some baselines in linear overparameterized models. 
The set-up of the numerical experiments is summarized as follows:
\begin{itemize}
    \item \textbf{Synthetic data generation}: For $p=1.5n$, orthogonal matrices $\bs U\in \mathbb{R}^{n \times n},\bs V \in \mathbb{R}^{p \times n}$ are i.i.d.~from the uniform distribution over the set of orthogonal matrices satisfying $\bs U^{\top}\bs U=\bs V^{\top}\bs V=\bs I_n$, respectively. Entries of $\bs \beta_0 \in \mathbb{R}^{p}$ are i.i.d.~from $\normal(0,1/p)$. 
    With given singular values $\{s_i\}_{i=1}^{n}$, the design matrix $\bs X:=\bs U\bs S\bs V^{\top}$ is computed. 
    $\bs y$ is generated $50$ times from $\normal_n(\bs X\bs \beta_0,\bs I_n)$, i.e., $\sigma_0^2=1$. 
    \item \textbf{Singular values}:
    We consider three different types of singular values: 
(i) $\bs X$ has a fixed intrinsic dimension $d_*=10$, i.e., $s_1=\cdots=s_{10}=n^{1/2}$ and $s_{11}=s_{12}=\cdots=s_n=0$, 
(ii) $s_i=n^{1/2}i^{-1}$, and
(iii) $s_i=n^{1/2}i^{-1/2}$. 
Theorem~\ref{theo:convergence_of_FV} proves the convergence of FV in settings (i) and (ii) since $\lim_{n \to \infty}n^{-1}\sum_{i=1}^{n}s_i^2<\infty$ in (i) and (ii), whereas 
it does not prove the convergence in setting (iii) since $\lim_{n \to \infty}n^{-1}\sum_{i=1}^{n}s_i^2$ is not bounded in (iii).
    \item \textbf{Evaluation}: 
    We evaluate the following baselines and LFV with $50$ experiments. 
    Throughout these experiments, we employ $\alpha=0.1$ for ridge regularization~(\ref{eq:ridge}). 
    
    \begin{enumerate}[{(a)}]
    \item \textbf{TIC} penalty~\citep{takeuchi1976distribution} is extended to the overparameterization setting~\citep{thomas2020interplay} by replacing the inverse of the Hessian matrix $\hat{\bs G}$ in TIC with the $\kappa$-generalized inverse matrix $\hat{\bs G}^{+}_{\kappa}$. 
    See Section~\ref{subsec:TIC} for the definition of $\text{TIC}(\kappa):=\tr\{\hat{\bs F}\hat{\bs G}^{+}_{\kappa}\}$. 
    \item \textbf{FV} is empirically computed with $T=15n$ samples of $\bs \beta$ generated from the quasi-posterior~(\ref{eq:posterior}). 
    Its expectation shown in Lemma~\ref{lem:explicit_difference} is also computed. 
    \item \textbf{LFV} is computed with $\delta = 1/(10n)$ and $T=15n$ Langevin samples $\bs \gamma$ (see eq.~(\ref{eq:Langevin})).
    \end{enumerate}
\end{itemize}

\begin{table}[!ht]
\centering
\caption{The generalization gap $\Delta(\alpha ; \wrtX)$ and its estimates (TIC, FV, LFV) with the standard deviations in setting (i) where $s_i=n^{1/2}$ for $i=1,2,\ldots,10$ and $0$ otherwise. The generalization gap and LFV estimate are gray-colored.}
\label{table:dstar=10}
\scalebox{1}{
\begin{tabular}{rllll}
\toprule 
 $n$ & 100 & 200 & 300 & 400 \\ 
\midrule 
\rowcolor{Gray}
 $\Delta(\alpha ; \wrtX)$ & 9.091 & 9.091 & 9.091 & 9.091 \\ 
\cmidrule(r){1-1} 
$\text{TIC}\,(\kappa=0)$ & 8.752 $\pm$ 1.284 & 9.272 $\pm$ 0.941 & 9.633 $\pm$ 0.829 & 9.536 $\pm$ 0.769 \\ 
$\text{TIC}\,(\kappa=0.1)$ & 8.752 $\pm$ 1.284 & 9.272 $\pm$ 0.941 & 9.633 $\pm$ 0.829 & 9.536 $\pm$ 0.769 \\ 
$\FV(\alpha;\wrtX)$ & 8.432 $\pm$ 1.180 & 8.658 $\pm$ 0.864 & 8.907 $\pm$ 0.761 & 8.781 $\pm$ 0.693 \\ 
$\mathbb{E}_{\bs y}[\FV(\alpha;\wrtX)]$ & 8.527 & 8.810 & 8.902 & 8.946 \\ 
\rowcolor{Gray}
$\LFV(\alpha;\wrtX)$ & 8.317 $\pm$ 1.144 & 8.815 $\pm$ 0.931 & 9.044 $\pm$ 0.8463 & 8.951 $\pm$ 0.743 \\ 
\bottomrule 
\end{tabular}
}
\end{table}

\begin{table}[!ht]
\centering
\caption{The generalization gap $\Delta(\alpha ; \wrtX)$ and its estimates (TIC, FV, LFV) with the standard deviations in setting (ii), where $s_i=n^{1/2}i^{-1}$. The generalization gap and LFV estimate are gray-colored.}
\label{table:nu=2}
\scalebox{1}{
\begin{tabular}{rllll}
\toprule 
 $n$ & 200 & 400 & 600 & 800 \\ 
\midrule 
\rowcolor{Gray}
 $\Delta(\alpha ; \wrtX)$ & 4.417 & 4.442 & 4.451 & 4.455 \\ 
\cmidrule(r){1-1} 
$\text{TIC}\,(\kappa=0)$ & 192.5 $\pm$ 18.01 & 390.1 $\pm$ 24.97 & 589.2 $\pm$ 35.37 & 787.6 $\pm$ 39.41 \\ 
$\text{TIC}\,(\kappa=0.1)$ & 135.6 $\pm$ 12.83 & 194.1 $\pm$ 12.41 & 239.8 $\pm$ 14.52 & 277.7 $\pm$ 13.92 \\ 
$\FV(\alpha;\wrtX)$ & 4.295 $\pm$ 0.415 & 4.302 $\pm$ 0.354 & 4.351 $\pm$ 0.306 & 4.369 $\pm$ 0.255 \\ 
$\mathbb{E}_{\bs y}[\FV(\alpha;\wrtX)]$ & 4.305 & 4.385 & 4.412 & 4.425 \\ 
\rowcolor{Gray}
$\LFV(\alpha;\wrtX)$ & 3.985 $\pm$ 0.491 & 4.193 $\pm$ 0.467 & 4.272 $\pm$ 0.357 & 4.341 $\pm$ 0.305 \\ 
\bottomrule 
\end{tabular}
}
\end{table}

\begin{table}[!ht]
\centering
\caption{
The generalization gap $\Delta(\alpha ; \wrtX)$ and its estimates (TIC, FV, LFV) with the standard deviations in setting (iii), where $s_i=n^{1/2}i^{-1/2}$. The generalization gap and LFV estimate are gray-colored.}
\label{table:nu=1}
\scalebox{1}{
\begin{tabular}{rlllll}
\toprule 
 $n$ & 200 & 400 & 600 & 800 \\ 
\midrule 
\rowcolor{Gray}
 $\Delta(\alpha ; \wrtX)$ & 29.98 & 36.66 & 40.63 & 43.46 \\ 
\cmidrule(r){1-1} 
$\text{TIC}\,(\kappa=0)$ & 149.7 $\pm$ 14.68 & 335.3 $\pm$ 21.68 & 527.6 $\pm$ 33.40 & 721.0 $\pm$ 37.64 \\ 
$\text{TIC}\,(\kappa=0.1)$ & 149.7 $\pm$ 14.68 & 335.3 $\pm$ 21.68 & 527.6 $\pm$ 33.40 & 721.0 $\pm$ 37.64 \\ 
$\FV(\alpha;\wrtX)$ & 24.68 $\pm$ 2.250 & 32.32 $\pm$ 2.026 & 37.06 $\pm$ 2.292 & 40.33 $\pm$ 2.046 \\ 
$\mathbb{E}_{\bs y}[\FV(\alpha;\wrtX)]$ & 24.78 & 32.60 & 37.26 & 40.53 \\ 
\rowcolor{Gray}
$\LFV(\alpha;\wrtX)$ & 22.32 $\pm$ 2.092 & 30.77 $\pm$ 2.100 & 35.83 $\pm$ 2.346 & 39.42 $\pm$ 2.052 \\ 
\bottomrule  
\end{tabular}
}
\end{table}

\bigskip 
The results are presented in the following Table~\ref{table:dstar=10}--\ref{table:nu=1}. Overall, FV and LFV were able to more effectively estimate the generalization gap $\Delta(\alpha ; \wrtX)$ ($p=1.5n,n\to \infty$), as compared to TIC, and their biases were not drastically different.
TIC was entirely inaccurate in settings (ii) and (iii); together with the result of (i), TIC was able to estimate the number of nonzero eigenvalues in the Hessian matrix $\hat{\bs G}=n^{-1}\bs X^{\top}\bs X$ (i.e., the number of nonzero singular values of $\bs X$), which is generally different from the generalization gap $\Delta(\alpha ; \wrtX)$.

\begin{remark}[Computational time]
\label{remark:computational_complexity}
Let us mention computational time of FV and LFV.
We check the computation time
by the linear synthetic dataset with $n=100,p=2000,\alpha=0.01,s_i=n^{1/2}i^{-1/2},T=500$ five times.
Among five experiments, 
the processing time for LFV is $5.33 \pm 0.09$ seconds,
and that for FV is $18.87 \pm 0.18$ seconds,
which implies that LFV is in fact faster than FV.
For these experiments, we used AMD Ryzen 7 5700X processors. The computation is not parallelized for fair comparison. 
\end{remark}

\subsection{LFV for non-linear models}
\label{subsec:experiments_nn_synthetic}

This subsection presents the evaluation of the Langevin FV for non-linear neural networks via synthetic dataset experiments. We use the procedure in  Section~\ref{subsec:non-linear} in applying LFV to non-linear models. 
The set-up is summarized as follows:
\begin{itemize} 
\item \textbf{Synthetic data}: 
Let $n \in \{50,500,1000\},\sigma^2=1$. 
For $i=1,2,\ldots,n$, we generate $\bs z_i \overset{\text{i.i.d.}}{\sim} \normal_d(0,\bs I_d)$
and
$\varepsilon_i \overset{\text{i.i.d.}}{\sim} \normal(0,\sigma^2)$,
and set
$\mu_i:=\mu(\bs z_i)$ and $y_i:=\mu_i+\varepsilon_i$ with a function $\mu(\bs z):=3\tanh(\langle \bs z,\bs 1 \rangle/2)$. 

\item \textbf{Neural network~(NN)}: We employ a fully-connected one-hidden-layer NN:
\[ 
g_{\bs \theta}(\bs z):=\langle \bs \theta^{(2)}, \tanh(\bs \theta^{(1)}\bs z+\bs \theta^{(0)})\rangle
\]
with $M \in \{50,100,150\}$ hidden units, where $\bs \theta=(\bs \theta^{(0)},\bs \theta^{(1)},\bs \theta^{(2)}) \in \mathbb{R}^M \times \mathbb{R}^{M \times d} \times \mathbb{R}^M$ is a parameter vector (and thus the number of parameters is $p=M(d+2)$) and 
the function
$\tanh(\cdot)$ applies the hyperbolic tangent function $\tanh(z):=(\exp(z)-\exp(-z))/(\exp(z)+\exp(-z))$ entry-wise. 
Let $\tilde{\bs \theta}_0$ be a parameter satisfying 
$g_{\tilde{\bs \theta}_0}(\bs z)
=
\mu(\bs z)$. 
For each experiment, we employ the true parameter $\bs \theta_0$ that is i.i.d.~drawn from the element-wise independent normal distribution with mean $\tilde{\bs \theta}_0$ and the element-wise variance $0.01$. 
Further, we initialize the parameter $\bs \theta$ of the NN to be trained by the element-wise independent normal distribution whose mean is $\bs \theta_0$ with the larger element-wise variance $1$, and update $\bs \theta$ by a full-batch gradient descent with a learning rate $0.1$, ridge-regularization coefficient $\alpha=10^{-3}$ and, $0.3n$ iterations. 

\item \textbf{Langevin FV}: For each setting $(d,M) \in \{5,10,15\} \times \{50,100,150\}$, we take the average of LFV over $25$ experiments. 
In each experiment, we randomly generate $\{(\bs z_i,y_i)\}_{i=1}^{n} \subset \mathbb{R}^d \times \mathbb{R}$, train the NN $g_{\bs \theta}$, and compute $T = 3000$ iterations of the Langevin process with $\delta=10^{-5}$. 
We discard the first $0.1T$ iterations, and employ the remaining $0.9T$ iterations to compute LFV.

\item \textbf{Generalization gap $\tilde{\Delta}$}: 
For each setting $(d,M) \in \{5,10,15\} \times \{50,100,150\}$, 
we take average of the following generalization gap
\begin{align*}
\tilde{\Delta}
&:=
\frac{1}{2\sigma_0^2}
\left\{
\mathbb{E}_{\bs y^*}\left(
\frac{1}{n}\sum_{i=1}^{n}
\left\{
    y^*_{i}-\bs g_{\hat{\bs \theta}}(\bs z_{i})
\right\}^2
\right)
-
\frac{1}{n}\sum_{i=1}^{n}
\left\{
    y_{i}-\bs g_{\hat{\bs \theta}}(\bs z_{i})
\right\}^2
\right\}
\end{align*}
 over $200$ times experiments. 
 In each experiment, we randomly generate $\{(\bs z_i,y_i)\}_{i=1}^{n} \subset \mathbb{R}^d \times \mathbb{R}$, train the NN $g_{\bs \theta}$, and evaluate $\tilde{\Delta}$ by leveraging the ground-truth $\mu_i=\mu(\bs z_i)$ and $\sigma_0^2=1$. 
\end{itemize}

\begin{table}[!ht]
\centering 
\caption{The generalization gap and LFV for the NN model with $n=50$ and $T=3000$. LFV values for the overparameterized regime (i.e., $p=M(d+2)>n$) are gray-colored.}
\label{table:NN_experiment_n50}
\begin{tabular}{rcccccc}
\toprule 
& \multicolumn{2}{c}{$M=50$} & \multicolumn{2}{c}{$M=100$} & \multicolumn{2}{c}{$M=150$} \\ 
& LFV & $\tilde{\Delta}$ & LFV & $\tilde{\Delta}$ & LFV & $\tilde{\Delta}$ \\
\midrule 
$d=5$  & \cellcolor{Gray} $2.90 \pm 0.84$ & $5.78$  & \cellcolor{Gray} $3.50 \pm 1.13$ & $5.84$  & \cellcolor{Gray} $3.85 \pm 0.90$ & $5.88$ \\
$d=10$ & \cellcolor{Gray} $3.98 \pm 1.22$ & $11.66$ & \cellcolor{Gray} $4.25 \pm 1.16$ & $11.80$ & \cellcolor{Gray} $5.79 \pm 2.03$ & $11.95$ \\
$d=15$ & \cellcolor{Gray} $4.73 \pm 1.70$ & $15.17$ & \cellcolor{Gray} $6.23 \pm 1.65$ & $15.55$ & \cellcolor{Gray} $7.28 \pm 1.80$ & $15.94$ \\
\bottomrule 
\end{tabular} 
\end{table}

\begin{table}[!ht]
\centering 
\caption{The generalization gap and LFV for the NN model with $n=500$ and $T=3000$. LFV values for the overparameterized regime (i.e., $p=M(d+2)>n$) are gray-colored.}
\label{table:NN_experiment_n500}
\begin{tabular}{rcccccc}
\toprule 
& \multicolumn{2}{c}{$M=50$} & \multicolumn{2}{c}{$M=100$} & \multicolumn{2}{c}{$M=150$} \\ 
& LFV & $\tilde{\Delta}$ & LFV & $\tilde{\Delta}$ & LFV & $\tilde{\Delta}$ \\
\midrule 
$d=5$ & $7.72 \pm 1.31$ & $8.85$ & \cellcolor{Gray} $8.42 \pm 1.18$ & $8.00$ & \cellcolor{Gray} $9.16 \pm 1.51$ & $8.48$ \\
$d=10$ & \cellcolor{Gray} $14.64 \pm 1.99$ & $17.79$ & \cellcolor{Gray} $15.79 \pm 1.69$ & $18.97$ & \cellcolor{Gray} $17.58 \pm 2.11$ & $19.56$ \\
$d=15$ & \cellcolor{Gray} $20.34 \pm 1.79$ & $29.85$ & \cellcolor{Gray} $24.68 \pm 2.07$ & $30.86$ & \cellcolor{Gray} $26.32 \pm 2.82$ & $31.24$ \\
\bottomrule 
\end{tabular} 
\end{table}

\begin{table}[!ht]
\centering
\caption{The generalization gap and LFV for the NN model with $n=1000$ and $T=3000$. LFV values for the overparameterized regime (i.e., $p=M(d+2)>n$) are gray-colored.}
\label{table:NN_experiment_n1000}
\begin{tabular}{rcccccc}
\toprule 
& \multicolumn{2}{c}{$M=50$} & \multicolumn{2}{c}{$M=100$} & \multicolumn{2}{c}{$M=150$} \\ 
& LFV & $\tilde{\Delta}$ & LFV & $\tilde{\Delta}$ & LFV & $\tilde{\Delta}$ \\
\midrule 
$d=5$ & $8.86 \pm 1.20$ & $9.13$ & $9.99 \pm 1.32$ & $9.80$ & \cellcolor{Gray} $10.78 \pm 1.38$ & $10.15$ \\
$d=10$ & $17.16 \pm 1.56$ & $23.46$ & \cellcolor{Gray} $19.09 \pm 1.55$ & $23.87$ & \cellcolor{Gray} $21.50 \pm 2.41$ & $23.99$ \\
$d=15$ & $25.42 \pm 2.04$ & $31.70$ & \cellcolor{Gray} $30.25 \pm 2.21$ & $31.82$ & \cellcolor{Gray} $32.81 \pm 2.30$ & $31.87$ \\
\bottomrule 
\end{tabular} 
\end{table}

LFV and generalization loss $\tilde{\Delta}$ for a non-linear NN with $M$ hidden units are shown in Tables~\ref{table:NN_experiment_n50}--\ref{table:NN_experiment_n1000}. 
LFV values for overparameterized models (i.e., $p=M(d+2)>n$) are in gray.
It was observed that the value of LFV provides a rough estimate of the generalization gap even for non-linear overparameterized NN models, though our overparameterized theories on FV and LFV are proved for linear models. 
Table \ref{table:NN_experiment_n50} ($n=50$) and Table \ref{table:NN_experiment_n1000} ($n=1000$) suggested that the biases of LFV reduces as the sample size grows.

\subsection{LFV for non-linear models using real datasets}
\label{subsec:experiments_nn_real}

In this subsection, 
we compare LFV to cross validation statistic using real datasets with small sample sizes, where we note that the cross validation becomes prohibited as the sample size becomes larger.
We use the procedure in Section~\ref{subsec:non-linear} in applying LFV to non-linear models. 
The set-up is summarized as follows.

\begin{itemize}
\item \textbf{Real data:} 
We collected the following 7 regression datasets from KEEL dataset repository~\citep{alcala2011keel}: 
``machineCPU'' ($n=208, d=6$), 
``wankara'' ($n=320, d=9$), 
``baseball'' ($n=336, d=16$), 
``dee'' ($n=364,d=6$), 
``autoMPG6'' ($n=391, d=5$), 
``autoMPG8'' ($n=391, d=7$), 
and ``stock'' ($n=949, d=9$). 
These datasets are selected so that our NN structure described below becomes overparameterized. 
For each dataset, we standardized (scaling and centering) the design matrix $\bs X$ and the target variable $\bs y$.

\item \textbf{Neural network~(NN)}: We employ the same architecture as the non-linear NN in Section~\ref{subsec:experiments_nn_synthetic}, with the number of hidden layers $M=100$. 
For the neural netweork training, we first initialize the parameter $\bs \theta$ randomly by the normal distribution $\normal_p(\bs 0,d^{-1/2}\bs I_p)$; we employ $20$ different random initializations for each dataset experiment. 
We train the NN by the fullbatch gradient descent (using the entire dataset) with the learning rate $0.01$, ridge-regularization coefficient $\alpha=0.1$. 
Gradient descent algorithm is terminated if $|\text{loss}_{t}-\text{loss}_{t-1}|/|\text{loss}_{t-1}|<10^{-6}$, where $\text{loss}_t$ denotes the training loss at the iteration $t$.

\item \textbf{Cross-validation~(CV)}: We divide each dataset into the training set ($90\%$) and test set ($10\%$) uniformly randomly. 
Starting from the NN parameters already trained with the entire dataset as shown above, we train NN by gradient descent with the learning rate $0.001$ using the divided training set. 
Gradient descent algorithm for the $j$th CV instance is terminated if 
$|\text{loss}_{j,t}-\text{loss}_{j,t-1}|/|\text{loss}_{j,t-1}|<10^{-6}$, where $\text{loss}_{j,t}$ denotes the training loss at the iteration $t$, using the training set in the $j$th CV instance. 
After the training, we compute 10-fold cross validation statistic: \begin{align*}
    \CV:= 
    \frac{1}{2\hat{\sigma}_0^2}
    \frac{1}{10}
    \sum_{j=1}^{10}
&\bigg\{
    \frac{1}{n_{\text{test}(j)}} 
    \sum_{i=1}^{n_{\text{test}(j)}} 
    \{y_{\text{test}(j),i}-g_{\hat{\bs \theta}_{\text{train}(j)}}(\bs x_{\text{test}(j),i})\}^2,
    \bigg\}
\end{align*}
where $\{(\bs x_{\text{test}(j),i},y_{\text{test}(j),i})\}_{i=1}^{n_{\text{test}}}$ denotes the $j$-th test set and $\{(\bs x_{\text{train}(j),i},y_{\text{train}(j),i})\}_{i=1}^{n_{\text{train}}}$ denotes the $j$-th training set. 
$g_{\hat{\bs \theta}_{\text{train}(j)}}$ is the NN trained with the $j$-th training set. 
Since the training of the NN was unstable in some of the training sets, we ignored a few instances of 10 training sets in CV.

\item \textbf{Langevin FV:} After training the NN $g_{\bs \theta}$, we compute $T=10000$ iterations of the Langevin process with
$\delta=5 \times 10^{-6}/\hat{\sigma}_0^2$, where $\hat{\sigma}_0^2:=(n-1)^{-1}\sum_{i=1}^{n} 
    \{y_{i}-g_{\hat{\bs \theta}}(\bs x_i)\}^2$ denotes the estimator of $\sigma_0^2$ for each dataset. 
    We discard first $2500$ iterations and compute LFV using the remaining $7500$ iterations; 
    we further compute 
    \[
        \underbrace{
        \frac{1}{2\hat{\sigma}_0^2}
        \sum_{i=1}^{n}
            \{y_{i}-g_{\hat{\bs \theta}_{\text{train}}}(\bs x_{i})\}^2
        }_{(\text{training loss})}
        +
        \LFV(\alpha;\X),
    \]
    where it can be regarded as a Langevin variant of WAIC~\citep{watanabe2010asymptotic}, and compare it to the 
    CV statistic approximating the generalization loss.

\item \textbf{Experimental Setting:} we compute $\CV$ statistic and training loss+LFV for each dataset. 
More precisely, we compute the average and standard deviation of these values over different $20$ random initializations.

\end{itemize}

Results are shown in Table~\ref{table:real}. 
Overall, training loss+LFV roughly approximated the CV statistic; they are expected to be compatible as both of these values approximate the generalization loss. 
Further, we computed the Spearman's rank correlation between the averaged training loss+LFV and the averaged CV statistic (over $20$ random initializations), for $7$ datasets. 
The correlation was $0.75$, while that between the training loss (without LFV) and CV statistic was $-0.75$. 

In addition to the computational burden, 
we note that CV statistic became unstable even in these experiments, as the multiple NNs trained with different CV instances fell in different local optima. 
So we ignored several CV instances for more stable computation; LFV was more stably computed as LFV was computed with only a single training of the NN.

Therefore, we expect that the proposed LFV can be used as a substitute of the cross validation statistic for real datasets, even with larger sample sizes.

\begin{table}
\centering
\caption{CV statistic and training loss+LFV over 20 different initializations. 
$n$ denotes the number of samples, and $p=M(d+2)$ denotes the number of parameters used in NN. 
}
\label{table:real}
\begin{tabular}{lccc|ccc}
\hline
Dataset & $n$ & $p$ & $\CV$ & Training loss+LFV\\
\hline
machineCPU & 208 & 800 & $0.637 \pm 0.034$ & $0.646 \pm 0.123$ \\
wankara & 320 & 1100 & $0.417 \pm 0.019$ & $0.560 \pm 0.068$ \\
baseball & 336 & 1800 & $0.671 \pm 0.011$ & $0.661 \pm 0.103$ \\
dee & 364 & 800 & $0.484 \pm 0.025$ & $0.622 \pm 0.107$\\
autoMPG6 & 391 & 700 & $0.523 \pm 0.045$ & $0.610 \pm 0.101$\\
autoMPG8 & 391 & 900 & $0.431 \pm 0.055$ & $0.559 \pm 0.056$\\
stock & 949 & 1100 & $0.526 \pm 0.016$ & $0.580 \pm 0.122$\\
\hline
\end{tabular}
\end{table}

\section{Conclusion}
\label{sec:conclusion}
In this paper, we considered a Gibbs generalization gap estimation method for overparameterized models. 
We proved that FV works as an asymptotically unbiased estimator, even in an overparameterization setting. 
We proposed a Langevin approximation of FV for efficient computation and applied it to overparameterized linear regression and non-linear neural network models. 

\section*{Acknowledgments}
We thank the editor, the AE, and two anonymous reviewers for constructive comments and suggestions.
We also thank Eiki Shimizu for suggesting several references, and 
Tetsuya Takabatake and Yukito Iba for helpful discussions.

\section*{Supplementary material}
Supplementary material contains the proofs of theorems.

\clearpage

\appendixpageoff
\appendixtitleoff
\renewcommand{\appendixtocname}{Supplementary material}
\begin{appendices}

\noindent {\Large \textbf{Supplementary material:}}
\begin{center}
{\Large A generalization gap estimation for overparameterized models} \\
\vspace{0.5em}
{\Large via the Langevin functional variance} \\
\vspace{1.5em}
{\large Akifumi Okuno and Keisuke Yano}
\end{center}

\section{Proof of the main theorems}
\label{sec:mainproof}
This section presents the proof of the main theorems. 
We begin by proving Theorem \ref{theo:convergence_of_FV_Gaussian} under the assumption that Theorem \ref{theo:convergence_of_FV} holds.
We then provide the supporting lemmas and prove Theorem \ref{theo:convergence_of_FV}.

\subsection{Proof of Theorem~\ref{theo:convergence_of_FV_Gaussian}}
\label{sec:proof_of_theo:convergence_of_FV_Gaussian}

With the $i$-th diagonal entry $\sigma_{i}^2$ of $\bs \Sigma$, write $\tilde{\mu} := \sum_{i=1}^{p} \sigma_i^2 \: (=\tr\{\bs \Sigma\})$ and 
    $\tilde{\sigma}^2 := 2 \sum_{i=1}^{p} \sigma_i^4 (\le 2(\tr\{\bs \Sigma\})^2)$. 
As $\eta=n^{-1}\tr\{\bs X^{\top}\bs X\} = n^{-1}\sum_{i=1}^{n}\sum_{j=1}^{p}x_{ij}^2=n^{-1}\sum_{i=1}^{n}\ell_i$ holds for $\ell_i:=\|\bs x_i\|_2^2$ that are i.i.d.~from a distribution with mean $\tilde{\mu}$ and variance $\tilde{\sigma}^2$, 
Chebyshev's inequality proves \[\mathbb{P}(|\eta-\tilde{\mu}|>\tilde{\sigma}) \le 1/n.\] 
Together with Theorem~\ref{theo:convergence_of_FV} and $\psi=1$ for Gaussian covariates, this yields
\begin{align*}
&    \mathbb{P}_{\wrtX}\left( 
            |\mathbb{E}_{\bs y}[\FV(\alpha ; \wrtX)]-\Delta(\alpha ; \wrtX)| > \varepsilon
        \right) \\
&\hspace{6em}\le 
    \mathbb{P}_{\wrtX}\left( 
        |\mathbb{E}_{\bs y}[\FV(\alpha ; \wrtX)]-\Delta(\alpha ; \wrtX)| 
        > \varepsilon
        \, \mid 
        |\eta - \tilde{\mu}| \le \tilde{\sigma}
    \right) 
    \underbrace{\mathbb{P}_{\wrtX}(|\eta - \tilde{\mu}| \le \tilde{\sigma})}_{\le 1}
     \\
&\hspace{9em}+
    \underbrace{\mathbb{P}_{\wrtX}\left( 
        |\mathbb{E}_{\bs y}[\FV(\alpha ; \wrtX)]-\Delta(\alpha ; \wrtX)| 
        > \varepsilon
        \, \mid 
        |\eta - \tilde{\mu}| > \tilde{\sigma}
    \right)}_{\le 1}
    \mathbb{P}_{\wrtX}(|\eta - \tilde{\mu}|>\tilde{\sigma}) \\
&\hspace{6em}\le 
    \frac{C'}{n}
    \left(
        \frac{(\tilde{\mu} + \tilde{\sigma})^2}{\varepsilon\alpha^2}
        +
        \frac{(\tilde{\mu} + \tilde{\sigma})^4}{\varepsilon\alpha^4}
        +
        \frac{(\tilde{\mu} + \tilde{\sigma})^4 b^4}{\varepsilon^2 \alpha^2 \sigma_0^4}
    \right)
    +
    \frac{1}{n}
\end{align*}
for an absolute constant $C'>0$. 
Considering the inequality $\tilde{\mu}
+
\tilde{\sigma}
\le 
\tr\{\bs \Sigma\}
+
\sqrt{2}\tr\{\bs \Sigma\}
\le
(1+\sqrt{2})\tr\{\bs \Sigma\}
=(1+\sqrt{2}) \xi$, we conclude the assertion. 
\qed

\subsection{Supporting lemmas}

We present the supporting lemmas for the proof of Theorem \ref{theo:convergence_of_FV}.
The first lemma provides the explicit form of the expectation of $(\bs v^{\top} \bs A \bs v)^{2}$.
The second lemma provides the tail bound of the trace of the Hadamard product.

\begin{lemma}
\label{prop:expectation_over_unit_sphere}
Let $n \in \mathbb{N}$ and let $\bs A \in \mathbb{R}^{n \times n}$ be a diagonal matrix. 
For a random vector $\bs v$ that is uniformly distributed over $\mathbb{S}^{n-1}$, we have 
\begin{align} 
\mathbb{E}[(\bs v^{\top}\bs A\bs v)^2]=\frac{2 \tr\{\bs A^2\} + (\tr\{\bs A\})^2}{n(n+2)}.
\label{eq:expectation_QF_over_sphere}
\end{align}
\end{lemma}

\begin{proof}[Proof of Lemma~\ref{prop:expectation_over_unit_sphere}]

Let $v_{i}$ be the $i$-th component of $v$.
Let 
\[\alpha_1 := \mathbb{E}[v_1^2 v_2^2]
\quad\text{and}\quad
\alpha_2 := \mathbb{E}[v_1^4].
\]
We begin by showing that
\begin{align}
\alpha_1 = \mathbb{E}[v_1^2 v_2^2]=\frac{1}{n(n+2)}.
\label{eq:alpha1}
\end{align}
Consider a random vector $\bs z=(z_1,z_2,\ldots,z_n) \in \mathbb{R}^n$ following the standard normal distribution $\normal_n(\bs 0,\bs I_n)$ and set $\tilde{\bs v}:=(\tilde{v}_1,\tilde{v}_2,\ldots,\tilde{v}_n):=\bs z/\|\bs z\|_2$. Then, the random vector $\tilde{\bs v}$ follows the uniform distribution over $\mathbb{S}^{n-1}$ (see, e.g., \citet{eaton1989group} Proposition 7.1). 
Since 
\[
\frac{1}{\|\bs z\|^{4}_{2}} = \int_{0}^{\infty}\exp(-t\|\bs z\|^{2}_{2})t\mathrm{d}t,
\]
we obtain
\begin{align*}
    \alpha_{1} &=     \mathbb{E}[\tilde{v}_1^2 \tilde{v}_2^2]
    =
    \mathbb{E}[z_1^2 z_2^2/\|\bs z\|_2^4] 
    =
    \mathbb{E}\left[ 
        z_1^2 z_2^2 \int_{0}^{\infty} \exp(-t\|\bs z\|_2^2) t \diff t
    \right].
\end{align*}
Together with the identity $\int_{0}^{\infty}\exp(-t\|\bs z\|^{2}_{2})t\mathrm{d}t=\int_{0}^{\infty}\prod_{i=1}^{n}\exp(-tz_{i}^{2})t\mathrm{d}t$, this yields
\begin{align*}
    &\alpha_{1}=
        \mathbb{E}\left[ 
        z_1^2 z_2^2 \int_{0}^{\infty} \prod_{i=1}^{n} \exp(-t z_i^2) t \diff t
    \right]\\
    &=
    \int_0^{\infty} 
    \mathbb{E}[z_1^2 \exp(-tz_1^2)]
    \mathbb{E}[z_2^2 \exp(-tz_2^2)]
    \prod_{i=3}^{n} \mathbb{E}[\exp(-tz_i^2)]
    t \diff t.
\end{align*}
Using function $u(t):=\mathbb{E}[\exp(-t z_1^2)]=(1+2t)^{-1/2}$, we have
\begin{align*}
    \alpha_1 
    =
    \int_0^{\infty} u'(t)^2 u(t)^{n-2} t \diff t 
    =
    \int_0^{\infty} (1+2t)^{-n/2-2} t \diff t,
\end{align*}
which yields 
\begin{align*}
    \alpha_1 =
    \left[ 
        \frac{1}{2n(1+2t)^{n/2}}
        -
        \frac{1}{2(n+2)(1+2t)^{n/2+1}}
    \right]_0^{\infty} 
    =
    \frac{1}{n(n+2)}
\end{align*}
and thus, we obtain (\ref{eq:alpha1}).

Taking expectation of $1=\|\bs v\|_2^4=\sum_{i \ne j} v_i^2 v_j^2 + \sum_{i} v_i^4$ yields equation $1 = n(n-1) \alpha_1 + n \alpha_2$, indicating that $\alpha_2=1/n-(n-1)\alpha_1$.
This, together with $\alpha_{1}=1/\{n(n+2)\}$, gives
\begin{align*}
\mathbb{E}[(\bs v^{\top}\bs A\bs v)^2]
=
\sum_{k} a_k^2 \mathbb{E}[v_k^4] 
+
\sum_{k \ne l} a_k a_l 
\mathbb{E}[v_k^2 v_l^2]
&=
\alpha_2 \sum_k a_k^2 
+
\alpha_1 \sum_{k \ne l} a_k a_l \\
&=
(\alpha_2-\alpha_1) \underbrace{\sum_k a_k^2}_{=\tr\{\bs A^2\}}
+
\alpha_1 \underbrace{\big(\sum_{k} a_k\big)^2}_{=(\tr\{\bs A\})^2}\\
&=\frac{2\tr\{\bs A^{2}\}}{n(n+2)}
+\frac{(\tr\{\bs A\})^2}{n(n+2)},
\end{align*}
which concludes the proof.
\end{proof}

\begin{lemma}
\label{prop:sq_sum}
    Let $n \in \mathbb{N}$ and 
    define a diagonal matrix $\bs R \in \mathbb{R}^{n \times n}$ whose diagonal entries are $\{r_i\}_{i=1}^{n}$, and let $\bs U=(\bs u_1,\bs u_2,\ldots,\bs u_n)$ be a random $n \times n$ orthogonal matrix where marginal density functions of the row vectors $\bs u_1,\bs u_2,\ldots,\bs u_n$ are $q_{1},q_{2},\ldots,q_{n}$, respectively. 
    Write
    \begin{align*}
    \psi := \max_{i=1,2,\ldots,n}\max_{\bs u \in \mathbb{S}^{n-1}} q_{i}(\bs u) \int_{\mathbb{S}^{n-1}} \diff \bs v.
    \end{align*}
    Then, it holds for $\bs T=(t_{ij}):=\bs U\bs R\bs U^{\top}$ and $\varepsilon>0$ that 
    \[
    \mathbb{P}\left( \tr\{\bs T \circ \bs T\} > \varepsilon \right) 
    <
    \frac{3\psi}{n\varepsilon}
     (\tr\{\bs R\})^2.
    \]
\end{lemma}
\begin{proof}[Proof of Lemma~\ref{prop:sq_sum}]
Write $\nu=\nu(n):=\int_{\mathbb{S}^{n-1}}\diff \bs v$. 
Since $t_{ii}=\bs u_i^{\top}\bs R\bs u_i$,
we have
\begin{align*} 
\max_{i=1,2,\ldots,n}
\mathbb{E}[t_{ii}^2]
=
\max_{i=1,2,\ldots,n}
\mathbb{E}[(\bs u_i^{\top}\bs R\bs u_i)^2] &=
\max_{i=1,2,\ldots,n}
\int_{\mathbb{S}^{n-1}} (\bs u_i^{\top}\bs R\bs u_i)^2 q_{i}(\bs u_i) \diff \bs u_i.
\\
&\le  
\left\{
\max_{i=1,2,\ldots,n}
\max_{\bs u \in \mathbb{S}^{n-1}} q_{i}(\bs u)
\nu
\right\}
\int_{\mathbb{S}^{n-1}} 
\left( 
(\bs v^{\top}\bs R\bs v)^2 
\frac{1}{\nu}
\right) \diff \bs v \\
&= 
\psi \, \mathbb{E}[(\bs v^{\top}\bs R\bs v)^2],
\end{align*}
where the random vector $\bs v$ follows the uniform distribution over $\mathbb{S}^{n-1}$. 
Together with Lemma \ref{prop:expectation_over_unit_sphere}, this gives 
\begin{align*}
\max_{i=1,2,\ldots,n}
\mathbb{E}[t_{ii}^2]
\le
\psi \, \frac{2\tr\{\bs R^2\} + (\tr\{\bs R\})^2}{n(n+2)}
\le 
\frac{\psi}{n^2}
\{2\tr\{\bs R^2\} + (\tr\{\bs R\})^2\}.
\end{align*}
Therefore, 
the Cauchy--Schwarz inequality for the trace yields
\begin{align}
    \mathbb{E}\left[\sum_{i=1}^{n}t_{ii}^2\right]
    \le 
    n \max_{i =1,2,\ldots, n} \mathbb{E}[t_{ii}^2]
    \le 
    \frac{\psi}{n}
     \{ 
    2\tr\{\bs R^2\} + (\tr\{\bs R\})^2
    \}
    \le
    \frac{3\psi}{n} (\tr\{\bs R\})^2.
    \label{eq:exp_sq}
\end{align}
Substituting the expectation bound (\ref{eq:exp_sq}) to the Markov inequality
gives
\[
\mathbb{P}\left(\sum_{i=1}^{n} t_{ii}^2 > \varepsilon\right)
<
\mathbb{E}\left[\sum_{i=1}^{n}t_{ii}^2\right]/\varepsilon,\]
which proves the assertion.
\end{proof}

\subsection{Proof of Theorem \ref{theo:convergence_of_FV}}

We first prepare notations for the proof. 
We denote by $\mathbb{P}_n$ the conditional probability $\mathbb{P}_{\wrtX}(\cdot \mid \eta)$. 
The decomposition $\bs X=\bs U\bs S\bs V^{\top}$ denotes a singular-value decomposition, where $\bs S$ is a diagonal matrix whose diagonal entries $s_{1} \ge s_{2} \ge \cdots \ge s_{n} \ge 0$ are singular values of $\bs X$. Therefore, we have $\eta=n^{-1}\tr\{\bs X^{\top}\bs X\}=n^{-1}\sum_{i=1}^{n} s_{i}^2$. 
The matrix $\bs H_{\alpha}$ is the regularized hat matrix in (\ref{eq:hatmatrix}).
Let
$\delta^{(1)}:=\tr\{\bs H_{\alpha}\circ \bs H_{\alpha}\},\delta^{(2)}:=\tr\{\bs H_{\alpha} \circ \bs H_{\alpha}^2\}$ and 
$\delta^{(3)}:=\tr
\{
\bs H_{\alpha} \circ ((\bs I_n-\bs H_{\alpha})(\bs X\bs \beta_0))^{\otimes 2}
\}/\sigma_0^2$, respectively.

Considering Lemma \ref{lem:explicit_difference}, 
we only have to prove tail bounds for $\delta^{(1)}$, 
$\delta^{(2)}$,
and $\delta^{(3)}$ via three steps.

\textbf{The first step: Tail bound for $\delta^{(1)}$}.
From the singular value decomposition $\bs X=\bs U \bs S \bs V^{\top}$, we have
\[\bs H_{\alpha}=\bs U\bs R^{(1)}\bs U^{\top},\] where $\bs R^{(1)}$ is the diagonal matrix whose
diagonal entries $r_1^{(1)},r_2^{(1)},\ldots,r_n^{(1)}$ are given by 
\[r_{i}^{(1)} := \frac{(s_{i}^2/n)}{(s_{i}^2/n) + \alpha}
\le \frac{(s_{i}^{2}/n)}{\alpha}.
\]
As the matrix $\bs R^{(1)}$ satisfies 
\[
\text{tr}\{\bs R^{(1)}\}=\sum_{i=1}^{n}r_i^{(1)} \le \frac{1}{\alpha} \cdot \frac{1}{n}\sum_{i=1}^{n} s_i^2 = \frac{1}{\alpha} \cdot \frac{1}{n} \{\text{tr}\bs X^{\top}\bs X\} = \frac{\eta}{\alpha},
\]
applying Lemma \ref{prop:sq_sum} to $\bs H_{\alpha}$ yields
\begin{align}
    \mathbb{P}_n\left(
        \delta^{(1)} > \varepsilon 
    \right)
\le 
    \frac{3\psi}{n\varepsilon} (\tr\{\bs R^{(1)}\})^2
\le
    \frac{3\psi}{n\varepsilon} \frac{\eta^2}{\alpha^2},
    \label{eq:delta1_concentration}
\end{align}
which completes the first step.

\textbf{The second step: Tail bound for $\delta^{(2)}$}.
From the singular value decomposition $\bs H_{\alpha}=\bs U \bs R^{(1)} \bs U^{\top}$,
we have
\[
\bs H_{\alpha}^{2}=\bs U \bs R^{(2)}\bs U^{\top},
\]
where $\bs R^{(2)}=(\bs R^{(1)})^2$, and so 
\[
\tr \{\bs R^{(2)}\} = \sum_{i=1}^{n}r_{i}^{(2)} \le \alpha^{-2}\sum_{i=1}^{\infty}(s_{i}^2)^2/n^2 \le \alpha^{-2}\left(n^{-1}\sum_{i=1}^{n} s_{i}^2\right)^2 = \alpha^{-2} \eta^2.
\]
Together with the bound from the Cauchy--Schwarz inequality for the trace
\[
    \tr \{\bs H_{\alpha} \circ \bs H_{\alpha}^2\}
    \le 
    (\tr\{\bs H_{\alpha} \circ \bs H_{\alpha}\})^{1/2}
    (\tr\{\bs H_{\alpha}^2 \circ \bs H_{\alpha}^2\})^{1/2}
    = (\delta^{(1)})^{1/2}
    (\tr\{\bs H_{\alpha}^2 \circ \bs H_{\alpha}^2\})^{1/2},
\]
Lemma \ref{prop:sq_sum} gives
\begin{align}
    \mathbb{P}_n(\delta^{(2)} > \varepsilon)
    \le 
    \mathbb{P}_n(\delta^{(1)} > \varepsilon \text{ or } \tr\{\bs H_{\alpha}^2 \circ \bs H_{\alpha}^2\} > \varepsilon)
    \le 
    \frac{3\psi}{n\varepsilon}
    \left( \frac{\eta^2}{\alpha^2} + \frac{\eta^4}{\alpha^4}\right),
    \label{eq:delta2_concentration}
\end{align}
which completes the second step.

\textbf{The third step: Tail bound for $\delta^{(3)}$}.
We begin by diagonalizing $\bs X^{\top}(\bs I-\bs H_{\alpha})^{2}\bs X$. The equation $(\bs I-\bs H_{\alpha})\bs X=\{\bs X-\bs X(\bs X^{\top}\bs X+n\alpha \bs I_p)^{-1}\bs X^{\top}\bs X\}=n\alpha \bs X(\bs X^{\top}\bs X+n\alpha \bs I_p)^{-1}$ indicates that 
\begin{align*}
    \bs X^{\top}(\bs I-\bs H_{\alpha})^2 \bs X
    &=
    (n\alpha)^2 (\bs X^{\top}\bs X+n\alpha \bs I_d)^{-1} \bs X^{\top}\bs X (\bs X^{\top}\bs X+n\alpha \bs I_p)^{-1} \\
&=
    (n\alpha)^2 (\bs X^{\top}\bs X+n\alpha \bs I_d)^{-1} \{\bs X^{\top}\bs X+n\alpha\} (\bs X^{\top}\bs X+n\alpha \bs I_p)^{-1} \\
    &\hspace{6em}
    -
    (n\alpha)^2 (\bs X^{\top}\bs X+n\alpha \bs I_d)^{-1} \{n\alpha\} (\bs X^{\top}\bs X+n\alpha \bs I_p)^{-1} \\
&=
    (n\alpha)^2(\bs X^{\top}\bs X+n\alpha)^{-1}
    -
    (n\alpha)^3(\bs X^{\top}\bs X+n\alpha \bs I_p)^{-2}.
\end{align*}
Thus, the $i$-th eigenvalue of $\bs X^{\top}(\bs I-\bs H_{\alpha})^2 \bs X$ is \[
\lambda_i:=\frac{(n\alpha)^2}{s_{i}^2+n\alpha}-\frac{(n\alpha)^3}{(s_{i}^2+n\alpha)^2}
=n\left\{\frac{\alpha^2}{(s_{i}^2/n+\alpha)}-\frac{\alpha^3}{(s_{i}^2/n+\alpha)^2}\right\}
=n\frac{\alpha^2 s_{i}^2/n}{(s_{i}^2/n+\alpha)^2}
\]
and we obtain
\[
\bs X^{\top}(\bs I-\bs H_{\alpha})^2 \bs X=\bs V\bs \Lambda \bs V^{\top},
\]
where $\bs \Lambda=\text{diag}(\lambda_1,\lambda_2,\ldots,\lambda_n)$.

Let 
$\tilde{\bs \beta}_0=(\tilde{\beta}_{01},\tilde{\beta}_{02},\ldots,\tilde{\beta}_{0p}):=\bs V^{\top}\bs \beta_0$. 
Since 
\[
\|\tilde{\bs \beta}_0\|_{\infty} \le \max_{1 \le i \le n}\|\bs v_i\|_{2}\|\bs \beta_0\|_{\infty} = \|\bs \beta_0\|_{\infty} = p^{-1/2}b,
\]
we have
\[
\bs \beta_0^{\top}\bs X^{\top}(\bs I-\bs H_{\alpha})^2 \bs X \bs \beta_0
=
\tilde{\bs \beta}_0^{\top}\bs \Lambda \tilde{\bs \beta}_0
=
\sum_{i=1}^{n} \lambda_i \tilde{\beta}_{0i}^2
\le 
n\frac{b^2}{p}
\sum_{i=1}^{n} \frac{\alpha^2 s_{i}^2/n}{(s_{i}^2/n+\alpha)^2} 
\le 
b^2 n^{-1}\sum_{i=1}^{n} s_{i}^2
=
b^2 \eta.
\]
Then, the Cauchy-Schwarz inequality for the trace yields
\begin{align*}
\delta^{(3)} 
&\le 
\left( \tr\{\bs H_{\alpha} \circ \bs H_{\alpha}\} \right)^{1/2}
\left( 
\tr\{
        ((\bs I-\bs H_{\alpha})\bs X\bs \beta_0)^{\otimes 2}
        \circ 
        ((\bs I-\bs H_{\alpha})\bs X\bs \beta_0)^{\otimes 2}
    \}
\right)^{1/2} \big/ \sigma_0^2\\
&\le 
( \delta^{(1)} )^{1/2}
b^2 \eta \big/ \sigma_0^2.
\end{align*}
This, together with (\ref{eq:delta1_concentration}), yields
\begin{align}
    \mathbb{P}_n(\delta^{(3)} > \varepsilon) 
&\le 
    \mathbb{P}_n\left(
     \delta^{(1)} > \varepsilon^2 \frac{\sigma_0^4}{b^4 \eta^2}
    \right) 
\le 
    \left( \varepsilon^2 \frac{\sigma_0^4}{b^4 \eta^2} \right)^{-1}
    \frac{3\psi}{n\varepsilon^{2}}
    \frac{\eta^2}{\alpha^{2}}
\le 
    \frac{3\psi}{n\varepsilon^{2}}
    \frac{\eta^4 b^4}{\alpha^2 \sigma_0^4},
    \label{eq:delta3_concentration}
\end{align}
which completes the third step.

Lastly, combining (\ref{eq:delta1_concentration})--(\ref{eq:delta3_concentration}) proves the assertion
\begin{align*}
    \mathbb{P}_n\left(
        |\mathbb{E}_{\bs y}(\FV(\alpha ; \wrtX)) - \Delta(\alpha ; \wrtX)|
        >
        \varepsilon
    \right)
&\le 
    \mathbb{P}_n(\delta^{(1)} > \varepsilon/3)
    +
    \mathbb{P}_n(\delta^{(2)} > \varepsilon/3)
    +
    \mathbb{P}_n(\delta^{(3)} > \varepsilon/3) \\
&\le 
    \frac{9\psi}{n\varepsilon} \left( 
        \frac{2\eta^2}{\alpha^2} + \frac{\eta^4}{\alpha^4}
    \right)
    +
    \frac{27\psi}{n\varepsilon^{2}}
    \frac{\eta^4 b^4}{\alpha^2 \sigma_0^4} \\
&\le 
    27
    \left(
        \frac{1}{n\varepsilon}\frac{\psi\eta^{2}}{\alpha^{2}}
        +
        \frac{1}{n\varepsilon}\frac{\psi\eta^4}{\varepsilon\alpha^4}
        +
        \frac{1}{n\varepsilon^{2}}
        \frac{\psi\eta^4 b^4}{\alpha^2 \sigma_0^4}
    \right),
\end{align*}
which completes the proof.
\qed

\clearpage

\section{Proof of Lemma \ref{lem:explicit_difference}}
\label{supp:proofoflemma}
We first provide useful formulae of the multivariate normal distribution
and then prove Lemma \ref{lem:explicit_difference}.
\subsection{Moments of Multivariate Normal Distribution}
\label{subsec:moments_of_mutlivarlate_normal}

Let $\bs z$ be a random vector following a $p$-variate normal distribution $\normal_p(\bs m,\bs \Sigma)$, and let 
$\tilde{\bs z}:=\bs z-\bs m \sim \normal_p(\bs 0,\bs \Sigma)$ be a centered random vector. 
Let $\bs a \in \mathbb{R}^{m}$ be any vector and $\bs B \in \mathbb{R}^{m \times p},\bs C=(c_{ij}) \in \mathbb{R}^{p \times p}$ be any matrices, for any $m \in \mathbb{N}$. 
Then, the following hold:

\begin{itemize}
    \item (The \textbf{second order} moment):
    We have
\begin{align}
    &\mathbb{E}[\|\bs a-\bs B\bs z\|_2^2]
    \nonumber\\
    &=
    \tr\{
        \bs a\bs a^{\top}
        -
        2\bs a(\bs B \mathbb{E}[\bs z])^{\top}
        +
        \bs B\mathbb{E}[\bs z\bs z^{\top}]\bs B^{\top}
    \}
    =
    \|\bs a-\bs B\bs m\|_2^2 
    +
    \tr\{\bs B\bs \Sigma \bs B^{\top}\},
    \label{eq:2nd_order_moment}
\end{align}
since
$\bs z$ satisfies $\mathbb{E}[\bs z\bs z^{\top}]=\bs m\bs m^{\top}+\bs \Sigma$.

\item (The \textbf{thrid order} moment): 
For $\bs x^{\otimes 3}:=\bs x \bs x^{\top}\bs x$,
we have
\begin{align}
    \mathbb{E}[\langle \bs a,(\bs B\tilde{\bs z})^{\otimes 3} \rangle]
    &=
    \langle \bs a, \bs B\mathbb{E}[\tilde{\bs z}\tilde{\bs z}^{\top}\bs B\bs B^{\top}\tilde{\bs z}] \rangle
    =
    0,
    \label{eq:3rd_order_moment}
\end{align}
since
$\mathbb{E}[\tilde{\bs z}\tilde{\bs z}^{\top}\bs C\tilde{\bs z}]=\bs 0$ as its $i$-th entry is \[\mathbb{E}\left[\sum_{j,k=1}^{p}c_{jk}\tilde{z}_i\tilde{z}_j\tilde{z}_k\right]=\sum_{j,k=1}^{p}c_{jk}\mathbb{E}[\tilde{z}_i\tilde{z}_j\tilde{z}_k]=0.\]

\item (The \textbf{fourth order} moment): 
We have
\begin{align}
    \mathbb{E}[\|\bs B \tilde{\bs z}\|_2^4]
&=
    \mathbb{E}[\tilde{\bs z}^{\top}\bs B^{\top}\bs B\tilde{\bs z}
    \tilde{\bs z}^{\top}\bs B^{\top}\bs B\tilde{\bs z}]
=
    2\tr\{(\bs B^{\top}\bs B\bs \Sigma)^2\}
    +
    (\tr\{\bs B^{\top}\bs B\bs \Sigma\})^2,
    \label{eq:4th_order_moment}
\end{align}
where the last equality follows from Section 8.2.4.~of \cite{cookbook}.
\end{itemize}

\subsection{Proof of Lemma~\ref{lem:explicit_difference}}
\label{sec:proof_of_prop:generalization_posterior_mean}

The proof consists of two parts: deriving an explicit form of $\Delta(\alpha ; \wrtX)$ and deriving that of $\mathbb{V}_{\bs \beta}[\log f(y_i \mid \bs x_i,\bs \beta)]$.

\textbf{The first part: Explicit form of $\Delta(\alpha ; \wrtX)$}.
Decompose $\Delta(\alpha ; \wrtX)$ as $\omega_{1}(\alpha)-\omega_{2}(\alpha)$,
where 
$\omega_1(\alpha):=
        \mathbb{E}_{\bs y,\bs y^*}
        \mathbb{E}_{\bs \beta}(\|\bs y^*-\bs X\bs \beta\|_2^2)$
and
$\omega_2(\alpha):=
        \mathbb{E}_{\bs y}
        \mathbb{E}_{\bs \beta}(\|\bs y-\bs X\bs \beta\|_2^2)$.

From the definition of $\omega_{1}(\alpha)$, we have
\begin{align*}
    \omega_1(\alpha)
    =   \mathbb{E}_{\bs y}
        \mathbb{E}_{\bs \beta}[
            n\sigma_0^2+\|\bs X(\bs \beta_0-\bs \beta)\|_2^2
        ]
    =
        n\sigma_0^2
        +
        \tr\left\{
        \bs X^{\top}\bs X
        \mathbb{E}_{\bs y}
        \mathbb{E}_{\bs \beta}\left[
            (\bs \beta_0-\bs \beta)^{\otimes 2}
        \right]
        \right\}.
\end{align*}
Since the quasi-posterior of $\beta$ is 
$\normal_{p}( \hat{\bs \beta}_{\alpha}, \bs Q_{\alpha})$ 
with $\hat{\bs \beta}_{\alpha}$ and $\bs Q_{\alpha}$ defined in (\ref{eq:ridge}) and (\ref{eq:Q}), respectively,
we have
    \begin{align*}
        \omega_1(\alpha)
    &=
        n\sigma_0^2 
        +
        n \cdot \tr\left\{
            \bs G
            \left[
            \bs Q_{\alpha}
            +
            \mathbb{E}_{\bs y}\left(
                (\bs \beta_0- \hat{\bs \beta}_{\alpha})^{\otimes 2}
            \right)
            \right]
        \right\} \\
        &=
        n\sigma_0^2 
        +
        n \cdot \tr\left\{
            \bs G
            \left[
            \bs Q_{\alpha}
            +
            \mathbb{E}_{\bs y}\left(
                (\bs \beta_0 - \sigma_0^{-2} \bs Q_{\alpha}\bs X^{\top}\bs y)^{\otimes 2}
            \right)
            \right]
        \right\},
    \end{align*}
    where $\bs G=n^{-1}\bs X^{\top}\bs X$.
    This, together with (\ref{eq:2nd_order_moment}), gives
    \begin{align*}
    \omega_{1}&=
        n\sigma_0^2
        +
        n \cdot \tr\{\bs G \bs Q_{\alpha}\}
        +
        n \cdot \tr\{
            \bs G[
                (\bs \beta_0- \sigma_0^{-2}\bs Q_{\alpha}\bs X^{\top}\bs X \bs \beta_0)^{\otimes 2}
                +
                \sigma_0^{-2}\bs Q_{\alpha}\bs X^{\top}\bs X \bs Q_{\alpha}
            ]
        \} \\
    &=
        n\sigma_0^2 
        +
        n \cdot \tr\{\bs G \bs Q_{\alpha}\}
        +
        n \cdot \tr\{\bs G
        ((\underbrace{\bs I_n-n\sigma_0^{-2}\bs Q_{\alpha}\bs G}_{=:\bs E_{\alpha}})\bs \beta_0)^{\otimes 2}\}
        +
        n^2\sigma_0^{-2}
        \tr\{(\bs G \bs Q_{\alpha})^2\} \\
    &=
        n\sigma_0^2 
        +
        n \cdot \tr\{\bs G \bs Q_{\alpha}\}
        +
        n^2 \sigma_0^{-2}
        \tr\{(\bs G \bs Q_{\alpha})^2\}
        +
        \tr\{(\bs X\bs E_{\alpha}\bs \beta_0)^{\otimes 2}\}.
        \end{align*}

From the definition of $\omega_{2}(\alpha)$, we have
    \begin{align*}
        \omega_2(\alpha)
    &=
        \tr
        \mathbb{E}_{\bs \beta,\bs y}[
            \bs y^{\otimes 2}
            -
            \bs y (\bs X\bs \beta)^{\top}
            -
            \bs X\bs \beta \bs y^{\top}
            +
            \bs X \bs \beta \bs \beta^{\top}\bs X^{\top}
        ] \\
    &=
        \tr
        \mathbb{E}_{\bs y}[
            \bs y^{\otimes 2}
            -
            \bs y (\bs X \sigma_0^{-2}\bs Q_{\alpha}\bs X^{\top}\bs y)^{\top}
            -
            \bs X\sigma_0^{-2}\bs Q_{\alpha}\bs X^{\top}\bs y \bs y^{\top} \\
            &\quad\quad\quad\quad
            +
            \bs X \{\bs Q_{\alpha}+(\sigma_0^{-2}\bs Q_{\alpha}\bs X^{\top}\bs y)^{\otimes 2}\} \bs X^{\top}
        ] \\
    &=
        \tr\{
            \bs X\bs Q_{\alpha}\bs X^{\top}
        \}
        +
        \tr\{
            \mathbb{E}_{\bs y}(
            (\bs y-\sigma_0^{-2}\bs X\bs Q_{\alpha}\bs X^{\top}\bs y)^{\otimes 2})
        \} \\
    &=
        \tr\{\bs X^{\top}\bs X \bs Q_{\alpha}\}
        +
        \tr\mathbb{E}_{\bs y}(\{\underbrace{(\bs I_n-\sigma_0^{-2}\bs X\bs Q_{\alpha}\bs X^{\top})}_{=:\bs F_{\alpha}}\bs y\}^{\otimes 2}).
    \end{align*}
Further, using $\bs G$, we have
    \begin{align*}
    \omega_{2}(\alpha)&=
        n \cdot \tr\{\bs G \bs Q_{\alpha}\}
        +
        \tr\{\bs F_{\alpha}[\sigma_0^2 \bs I_n+(\bs X\bs \beta_0)^{\otimes 2}]\bs F_{\alpha}\} \\
    &=
        n \cdot \tr\{\bs G \bs Q_{\alpha}\}
        +
        \sigma_0^2 \tr\{\bs F_{\alpha}^2\}
        +
        \tr\{(\bs F_{\alpha}\bs X \bs \beta_0)^{\otimes 2}\} \\
    &=
        n \cdot \tr\{\bs G \bs Q_{\alpha}\}
        +
        \sigma_0^2 
        \tr
        \{
            \bs I_n - 2\sigma_0^{-2}\bs X\bs Q_{\alpha}\bs X^{\top} 
            +
            \sigma_0^{-4}\bs X\bs Q_{\alpha}\bs X^{\top}\bs X\bs Q_{\alpha}\bs X^{\top}
        \} \\
        &\hspace{5em}
        +
        \tr\{(\bs X(\bs I_n-\sigma_0^{-2}\bs Q_{\alpha}\bs X^{\top}\bs X)\bs \beta_0)^{\otimes 2}\} \\
    &=
        n\sigma_0^2 
        + 
        n \cdot \tr\{\bs G\bs Q_{\alpha}\}
        -
        2 n \cdot \tr\{\bs G\bs Q_{\alpha}\}
        +
        n^2 \sigma_0^{-2}\tr\{(\bs G\bs Q_{\alpha})^2\}
        +
        \tr\{(\bs X\bs E_{\alpha}\bs \beta_0)^{\otimes 2}\} \\
    &=
        n \sigma_0^2 - n \cdot \tr\{\bs G\bs Q_{\alpha}\}
        +
        n^2 \sigma_0^{-2} \tr\{(\bs G\bs Q_{\alpha})^2\}
        +
        \tr\{(\bs X\bs E_{\alpha}\bs \beta_0)^{\otimes 2}\}.
\end{align*}

Therefore, we get
\begin{align*}
    \Delta(\alpha ; \wrtX)
&=
    \frac{1}{2\sigma_0^2}
    \{
    \omega_1(\alpha) 
    -
    \omega_2(\alpha)
    \}
=
    \frac{1}{2\sigma_0^2} \cdot 2n \cdot \tr\{\bs G\bs Q_{\alpha}\}
=
    \tr\bs H_{\alpha},
\end{align*}
which completes the first part.

\textbf{The second part: Explicit form of $\mathbb{V}_{\bs \beta}[\log f(y_i \mid \bs x_i,\bs \beta)]$}.
The second part is divided into three subparts (A,B,C).

\textbf{Subpart (A) of the second part}.
For  $m\in \mathbb{N}$, let $g_{m,i}:=\mathbb{E}_{\bs \beta}[(\bs x_i^{\top}\bs \beta)^m].$
Using $g_{m,i}$, we write the posterior expectation of $\log p(\bs x_i,y_i \mid \bs \beta)$ as
\begin{align*}
    \mathbb{E}_{\bs \beta}[\log f(y_i \mid \bs x_i,\bs \beta)]
    &=
    -\frac{1}{2\sigma_0^2}\big\{
    \mathbb{E}_{\bs \beta}[(\bs x_i^{\top}\bs \beta)^2]
    -
    2 y_i \mathbb{E}_{\bs \beta}[\bs x_i^{\top}\bs \beta]
    +
    y_i^2
    \big\} + C \\
    &=
    -\frac{1}{2\sigma_0^2}\{g_{2,i} - 2y_i g_{1,i} + y_i^2\} + C,
\end{align*}
where $C$ is a constant independent of $\bs \beta$ and $\bs x_{i}$.
Then, we can write the posterior variance as 
\begin{align*}
    \mathbb{V}_{\bs \beta}[\log f(y_i \mid \bs x_i,\bs \beta)]
    &=
    \frac{1}{(2\sigma_0^2)^2}
    \mathbb{E}_{\bs \beta}\left[
    \left\{
        (y_i-\bs x_i^{\top}\bs \beta)^2 
        -
        (g_{2,i}-2y_i g_{1,i} + y_i^2)
    \right\}^2
    \right] \\
&=
    \frac{1}{(2\sigma_0^2)^2}
    \mathbb{E}_{\bs \beta}\left[
    \left\{
        ((\bs x_i^{\top}\bs \beta)^2 - 2y_i (\bs x_i^{\top}\bs \beta) + y_i^2)
        -
        (g_{2,i} - 2 y_i g_{1,i} + y_i^2)
    \right\}^2
    \right] \\
&=
    \frac{1}{(2\sigma_0^2)^2}
    \mathbb{E}_{\bs \beta}\left[
    \left\{
        (\bs x_i^{\top}\bs \beta)^2 - 2y_i (\bs x_i^{\top}\bs \beta) 
        -
        (g_{2,i} - 2 y_i g_{1,i})
    \right\}^2
    \right].
    \end{align*}
Expanding the square yields
\begin{align*}
    \mathbb{V}_{\bs \beta}[\log f(y_i \mid \bs x_i,\bs \beta)]
&=
    \frac{1}{(2\sigma_0^2)^2}
    \big\{
        g_{4,i} 
        +
        [-4y_i]g_{3,i} 
        +
        [4y_i^2-2(g_{2,i}-2y_i g_{1,i})]g_{2,i} \\
    &\hspace{10em}
        +
        [4y_i(g_{2,i}-2 y_i g_{1,i})]g_{1,i}
        +
        [(g_{2,i}-2 y_i g_{1,i})^2]
    \big\} \\
&=
    \frac{1}{(2\sigma_0^2)^2}
    \big\{
        [4g_{2,i}-8g_{1,i}^2+4g_{1,i}^2]y_i^2
        \\
    &\qquad\qquad\quad +
        [-4g_{3,i}+4g_{1,i}g_{2,i}+4g_{1,i}g_{2,i}-4g_{1,i}g_{2,i}]y_i \\
    &\hspace{10em}
        +
        [g_{4,i}-2g_{2,i}^2+g_{2,i}^2]
    \big\} \\
&=
    \frac{1}{(2\sigma_0^2)^2}\big\{
        \underbrace{[4g_{2,i}-4g_{1,i}^2]}_{=:h_{2,i}}
        y_i^2
        +
        \underbrace{[-4g_{3,i}+4g_{1,i}g_{2,i}]}_{=:h_{1,i}}y_i
        +
        \underbrace{[g_{4,i}-g_{2,i}^2]}_{=:h_{0,i}}
    \big\},
\end{align*}
which completes subpart (A).

\textbf{Subpart (B) of the second part}.
We next find the expressions for 
$\mathbb{E}_{\bs y}[h_{2,i} y_i^2]$,
$\mathbb{E}_{\bs y}[h_{1,i} y_i]$, and
$\mathbb{E}_{\bs y}[h_{0,i}]$.
From Section 8.2 of \cite{cookbook}, we have 
\begin{align*}
    g_{1,i}
    &=
    \bs x_i^{\top}\hat{\bs \beta}_{\alpha}, \\
    g_{2,i}
    &=
    \bs x_i^{\top}
    \mathbb{E}_{\bs \beta}[\bs \beta \bs \beta^{\top}]\bs x_i
    =
    \bs x_i^{\top}\{\bs Q_{\alpha} + \hat{\bs \beta}_{\alpha}^{\otimes 2}\}\bs x_i, \\
    g_{3,i}
    &=
    \bs x_i^{\top}\mathbb{E}_{\bs \beta}(\bs \beta\bs x_i^{\top}\bs \beta \bs \beta^{\top})\bs x_i \\
    &=
    \bs x_i^{\top}[
        \hat{\bs \beta}_{\alpha}
        \bs x_i^{\top}
        (\bs Q_{\alpha}+\hat{\bs \beta}_{\alpha}^{\otimes 2})
        +
        (\bs Q_{\alpha}+\hat{\bs \beta}_{\alpha}^{\otimes 2})\bs x_i\hat{\bs \beta}_{\alpha}^{\top}
        +
        \bs x_i^{\top}\hat{\bs \beta}_{\alpha}
        (\bs Q_{\alpha}-\hat{\bs \beta}_{\alpha}^{\otimes 2})
    ]\bs x_i, \\
    &=
    (\bs x_i^{\top}\hat{\bs \beta}_{\alpha})
    \{\bs x_i^{\top}(\bs Q_{\alpha}+\hat{\bs \beta}_{\alpha}^{\otimes 2})\bs x_i\}
    +
    2(\bs x_i^{\top}\hat{\bs \beta}_{\alpha})\{\bs x_i^{\top}\bs Q_{\alpha}\bs x_i\}, \\
    g_{4,i}
    &=
    3\{\bs x_i^{\top}
    (\bs Q_{\alpha}+\hat{\bs \beta}_{\alpha}^{\otimes 2})\bs x_i\}^2
    -
    2(\bs x_i^{\top}\hat{\bs \beta}_{\alpha})^4.
\end{align*}
These lead to the following expressions for $h_{m,i} \: (m=0,1,2)$:
\begin{align*}
    h_{2,i}
    &=
    4g_{2,i}-4g_{1,i}^2
    =
    4\{\bs x_i^{\top}\bs Q_{\alpha}\bs x_i\}, \\
    h_{1,i}
    &=
    4g_1g_2 - 4g_3
    =
    -8(\bs x_i^{\top}\hat{\bs \beta}_{\alpha})\{\bs x_i^{\top}\bs Q_{\alpha}\bs x_i\}, \\
    h_{0,i}
    &=
    g_4-g_2^2
    =
    2\{\bs x_i^{\top}(\bs Q_{\alpha}+\hat{\bs \beta}_{\alpha}^{\otimes 2})\bs x_i\}^2
    -
    2(\bs x_i^{\top}\hat{\bs \beta}_{\alpha})^4 
    =
    4(\bs x_i^{\top}\hat{\bs \beta}_{\alpha})^2\{\bs x_i^{\top}\bs Q_{\alpha}\bs x_i\} 
    + 2\{\bs x_i^{\top}\bs Q_{\alpha}\bs x_i\}^2.
\end{align*}
Thus, using a vector $\bs e_i=(0,0,\ldots,0,1,0,\ldots,0) \in \{0,1\}^n$ whose $i$-th entry is $1$ and $0$ otherwise, 
we get
\begin{align}
    \mathbb{E}_{\bs y}[h_{2,i} y_i^2]
    &=
    4\{\bs x_i^{\top}\bs Q_{\alpha}\bs x_i\}\mathbb{E}_{\bs y}[y_i^2] \nonumber\\
    &=
    4\{\bs x_i^{\top}\bs Q_{\alpha}\bs x_i\}[\sigma_0^2 + \bs (\bs x_i^{\top}\bs \beta_0)^2], 
    \nonumber\\
    &=
    4\sigma_0^2 \{\bs x_i^{\top}\bs Q_{\alpha}\bs x_i\}
    +
    4 \bs x_i^{\top}\bs Q_{\alpha}\bs x_i \bs x_i^{\top}\bs \beta_0^{\otimes 2}\bs x_i, \label{eq:h2y2}
\end{align}
\begin{align}
    \mathbb{E}_{\bs y}[h_{1,i} y_i]
    &=
    -8 
    \sigma_0^{-2}
    \bs x_i^{\top}
    \bs Q_{\alpha}
    \bs X^{\top}\mathbb{E}_{\bs y}[\bs y \bs y^{\top}]\bs e_i \{\bs x_i^{\top}\bs Q_{\alpha}\bs x_i\} \quad 
    (\because \hat{\bs \beta}_{\alpha}=\sigma_0^{-2} \bs Q_{\alpha}\bs X^{\top}\bs y \text{ and }y_i=\bs y^{\top}\bs e_i) \nonumber\\
    &=
    -8 \sigma_0^{-2}\bs x_i^{\top}\bs Q_{\alpha}\bs X^{\top}[\sigma_0^2 \bs I_n + (\bs X \bs \beta_0)^{\otimes 2}]\bs e_i \{\bs x_i^{\top}\bs Q_{\alpha}\bs x_i\} \nonumber\\
    &=
    -8 \{\bs x_i^{\top}\bs Q_{\alpha}\bs x_i\}^2
    -8 \sigma_0^{-2} \bs x_i^{\top}\bs Q_{\alpha}\bs X^{\top} \bs X \bs \beta_0 \bs \beta_0^{\top}\bs x_i  \bs x_i^{\top}\bs Q_{\alpha}\bs x_i
    \quad (\because \bs X^{\top}\bs e_i=\bs x_i),\label{eq:h1y1}
\end{align}
and
\begin{align}
    \mathbb{E}_{\bs y}[h_{0,i}]
&=
    4 \sigma_0^{-4} \bs x_i^{\top}\bs Q_{\alpha}\bs X^{\top}\mathbb{E}_{\bs y}[\bs y \bs y^{\top}] \bs X \bs Q_{\alpha}\bs x_i \{ \bs x_i^{\top}\bs Q_{\alpha}\bs x_i\} 
    +
    2 \{\bs x_i^{\top}\bs Q_{\alpha}\bs x_i\}^2 \nonumber\\
&=
    4 \sigma_0^{-4} \bs x_i^{\top}\bs Q_{\alpha}\bs X^{\top}[\sigma_0^2 \bs I_n + (\bs X \bs \beta_0)^{\otimes 2}] \bs X \bs Q_{\alpha}\bs x_i \{ \bs x_i^{\top}\bs Q_{\alpha}\bs x_i\} 
    +
    2 \{\bs x_i^{\top}\bs Q_{\alpha}\bs x_i\}^2 \nonumber\\
&=
    4\sigma_0^{-2}  \bs x_i^{\top}\bs Q_{\alpha}\bs X^{\top} \bs X \bs Q_{\alpha}\bs x_i \bs x_i^{\top}\bs Q_{\alpha}\bs x_i \nonumber\\
&\hspace{3em}+
    4 \sigma_0^{-4} \bs x_i^{\top}\bs Q_{\alpha}\bs X^{\top}\bs X \bs \beta_0\bs \beta_0^{\top}\bs X^{\top}\bs X \bs Q_{\alpha}\bs x_i\bs x_i^{\top}\bs Q_{\alpha}\bs x_i \nonumber\\
&\hspace{6em}+
    2 \{\bs x_i^{\top}\bs Q_{\alpha}\bs x_i\}^2,\label{eq:h0}
\end{align}
which completes subpart (B).

\textbf{Subpart (C) of the second part}.
We lastly derive an expression of 
$\sum_{i=1}^{n}\mathbb{V}_{\bs \beta}[\log f(y_i \mid \bs x_i,\bs \beta)]$.
For any matrices $\bs A=(a_{ij})$ and $\bs B=(b_{ij}) \in \mathbb{R}^{p \times p}$, we denote by $L_{1}(\bs A)$
\[L_1(\bs A):=\sum_{i=1}^{n}\bs x_i^{\top}\bs A\bs x_i=\tr\{\bs X\bs A\bs X^{\top}\}\] 
and denote by $L_{2}(\bs A,\bs B)$
\[
L_{2}(\bs A,\bs B):=\sum_{i=1}^{n}\bs x_i^{\top}\bs A\bs x_i\bs x_i^{\top}\bs B\bs x_i=\tr\{(\bs X\bs A\bs X^{\top}) \circ (\bs X\bs B\bs X^{\top})\}.
\]

From (\ref{eq:h2y2}) and the identity $\bs X\bs Q_{\alpha}\bs X^{\top}=\sigma_0^2 \bs H_{\alpha}$, we get 
\begin{align*}
    \sum_{i=1}^{n}
    \mathbb{E}_{\bs y}[h_{2,i}y_i^2]
    &=
    4\sigma_0^2 L_1(\bs Q_{\alpha})
    +
    4L_2(\bs Q_{\alpha},\bs \beta_0^{\otimes 2})
    =
    4 \sigma_0^4 \tr\{\bs H_{\alpha}\}
    +
    4\sigma_0^{2} \tr\{\bs H_{\alpha} \circ (\bs X\bs \beta_0)^{\otimes 2}\}.
\end{align*}
From (\ref{eq:h1y1}), we get
\begin{align*}
    \sum_{i=1}^{n}
    \mathbb{E}_{\bs y}[h_{1,i}y_i]
    &=
    -8 L_2(\bs Q_{\alpha},\bs Q_{\alpha})
    -
    8\sigma_0^{-2}L_2(\bs Q_{\alpha}\bs X^{\top}\bs X\bs \beta_0\bs \beta_0^{\top},\bs Q_{\alpha}) \\
    &=
    -8 \sigma_0^4 \tr\{\bs H_{\alpha} \circ \bs H_{\alpha}\}
    -8 \sigma_0^2 \tr\{\bs H_{\alpha}(\bs X\bs \beta_0)^{\otimes 2} \circ \bs H_{\alpha}\}.
\end{align*}
Likewise, we have
\begin{align*}
    &\sum_{i=1}^{n}
    \mathbb{E}_{\bs y}[h_{0,i}]\\
    &=
    4\sigma_0^{-2}L_2(\bs Q_{\alpha}\bs X^{\top}\bs X\bs Q_{\alpha},\bs Q_{\alpha})
    +
    4\sigma_0^{-4}L_2(\bs Q_{\alpha}\bs X^{\top}\bs X\bs \beta_0 \bs \beta_0^{\top}\bs X^{\top}\bs X\bs Q_{\alpha},\bs Q_{\alpha})
    +
    2L_2(\bs Q_{\alpha},\bs Q_{\alpha}) \\
    &=
    4\sigma_0^4 \tr\{\bs H_{\alpha}^2 \circ \bs H_{\alpha}\} 
    +
    4\sigma_0^2 \{(\bs H_{\alpha} \bs X\bs \beta_0)^{\otimes 2} \circ \bs H_{\alpha}\}
    +
    2\sigma_0^4 \tr\{\bs H_{\alpha} \circ \bs H_{\alpha}\}.
\end{align*}
These identities yield
\begin{align*}
    \mathbb{E}_{\bs y}[\FV(\alpha ; \wrtX)]
    &=
    \mathbb{E}_{\bs y}\left[
    \sum_{i=1}^{n} \mathbb{V}_{\bs \beta}[\log f(y_i \mid \bs x_i,\bs \beta)] \right] \\
    &=
    \mathbb{E}_{\bs y}\left[
    \sum_{i=1}^{n} \frac{1}{2(\sigma_0^2)^2}\{h_{2,i} y_i^2 + h_{1,i}y_i + h_{0,i}\} \right] \\
    &=
    \frac{1}{(2\sigma_0^2)^2}
    \left\{
        \sum_{i=1}^{n}\mathbb{E}_{\bs y}[h_{2,i}y_i^2]
        +
        \sum_{i=1}^{n}\mathbb{E}_{\bs y}[h_{1,i}y_i]
        +
        \sum_{i=1}^{n}\mathbb{E}_{\bs y}[h_{0,i}]
    \right\} \\
    &=
    \tr\{\bs H_{\alpha}\}
    -
    \frac{3}{2}\tr\{\bs H_{\alpha} \circ \bs H_{\alpha}\}
    +
    \tr\{\bs H_{\alpha} \circ \bs H_{\alpha}^2\} \\
    &\hspace{6em}
    +
    \frac{1}{\sigma_0^2}\tr\{
        \bs H_{\alpha} 
        \circ 
        \underbrace{[
            (\bs X\bs \beta_0)^{\otimes 2}
            -
            2\bs H_{\alpha}(\bs X\bs \beta_0)^{\otimes 2} +
            (\bs H_{\alpha}\bs X\bs \beta_0)^{\otimes 2}
        ]}_{=((\bs I_n-\bs H_{\alpha})(\bs X\bs \beta_0))^{\otimes 2}}
    \},
\end{align*}
which completes subpart (C) of the second part and concludes the proof.
\qed

\section{Proof of Proposition~\ref{lem:expectation_of_J-FV}}
\label{supp:proof_of_lem:expectation_of_J-FV}

Since we have, for some constant $C$ independent of $\beta$,
\[
    \mathbb{E}_{\bs \beta}[\log f(\bs y \mid \bs X, \bs \beta)]
    =
    -\frac{1}{2\sigma_0^2}\|\bs y-\bs X\hat{\bs \beta}_{\alpha}\|_2^2
    - \frac{1}{2\sigma_0^2} \tr\{\bs X \bs Q_{\alpha}\bs X^{\top}\}
    +C,
\]
we get
\begin{align*}
    \mathbb{V}_{\bs \beta}[\log f(\bs y \mid \bs X, \bs \beta)]
&=
    \frac{1}{(2\sigma_0^2)^2}\mathbb{E}_{\bs \beta}\left[ 
        \{
        \|\bs y-\bs X\bs \beta\|_2^2
        -
        \|\bs y-\bs X\hat{\bs \beta}_{\alpha}\|_2^2
        -
        \tr\{\bs X \bs Q_{\alpha}\bs X^{\top}\}
        \}^2
    \right] \\
&=
    \frac{1}{(2\sigma_0^2)^2}\mathbb{E}_{\bs \beta}\left[ 
        \{ 
            \|\bs X\tilde{\bs \beta}\|_2^2
            -
            \langle \bs m,\bs X\tilde{\bs \beta}\rangle 
            -
            \zeta \}
        \}^2
    \right],
\end{align*}
where $\tilde{\bs \beta}:=\bs \beta-\hat{\bs \beta}_{\alpha},\bs m:=2(\bs y-\bs X\hat{\bs \beta}_{\alpha}),\zeta:=\tr\{\bs X\bs Q_{\alpha}\bs X^{\top}\}$.
Expanding the square gives    
\begin{align*}
&\mathbb{V}_{\bs \beta}[\log f(\bs y \mid \bs X, \bs \beta)]\\
&=
    \frac{1}{(2\sigma_0^2)^2}\mathbb{E}_{\bs \beta}\big[
        \|\bs X\tilde{\bs \beta}\|_2^4
        -
        2 \langle \bs m,(\bs X\tilde{\bs \beta})^{\otimes 3} \rangle 
        +
        \tilde{\bs \beta}^{\top}
        \{\bs X^{\top}(\bs m\bs m^{\top} - 2 \zeta \bs I)\bs X\} \tilde{\bs \beta} +
        2 \zeta \langle \bs m,\bs X\tilde{\bs \beta} \rangle 
        +
        \zeta^2
    \big].
\end{align*}
As $\tilde{\bs \beta}$ is centered, using (\ref{eq:3rd_order_moment}) and (\ref{eq:4th_order_moment}) leads to
\begin{align*}
    \mathbb{V}_{\bs \beta}[\log f(\bs y \mid \bs X,\bs \beta)]
    &=
    \frac{1}{(2\sigma_0^2)^2}
    \mathbb{E}_{\bs \beta}[
        \|\bs X\tilde{\bs \beta}\|_2^4
        +
        \tr
        \{
        (\bs X^{\top}(\bs m\bs m^{\top} - 2 \zeta \bs I)\bs X) \tilde{\bs \beta}\tilde{\bs \beta}^{\top}
        \}
        +
        \zeta^2
    ] \\
    &=
    \frac{1}{(2\sigma_0^2)^2}
    \bigg\{
        2\tr\{(\bs X\bs Q_{\alpha}\bs X^{\top})^2\}
        +
        (\tr\{\bs X\bs Q_{\alpha}\bs X^{\top}\})^2
        \\
        &\qquad\qquad\qquad
        +
        \tr\{(\bs m \bs m^{\top} - 2 \zeta \bs I) \bs X \bs Q_{\alpha}\bs X^{\top}\} + \zeta^2
    \bigg\}.
\end{align*}
From equations $\bs m=2(\bs I-\bs H_{\alpha})\bs y$, $\bs X\bs Q_{\alpha}\bs X^{\top}=\sigma_0^2\bs H_{\alpha}$ and $\zeta=\tr\{\bs X\bs Q_{\alpha}\bs X^{\top}\}$, we have
\begin{align*}
    &\mathbb{E}_{\bs y}[\mathbb{V}_{\bs \beta}[\log f(\bs y \mid \bs X, \bs \beta)]] \\
&=
    \frac{1}{(2\sigma_0^2)^2}
    \left\{
        2\tr\{(\sigma_0^2 \bs H_{\alpha})^2\}
        +
        \zeta^2
        +
        \sigma^2 \tr\{\mathbb{E}_{\bs y}[\bs m \bs m^{\top}] \bs H_{\alpha}\}
        -2\zeta^2
        +
        \zeta^2
    \right\} \\
&=
    \frac{1}{2}
    \tr \{ \bs H_{\alpha}^2 \}
    +
    \frac{1}{\sigma_0^2} \tr
    \{
    \mathbb{E}_{\bs y}[((\bs I-\bs H_{\alpha})\bs y)^{\otimes 2}] \bs H_{\alpha}
    \} + \underbrace{(\zeta^2 - 2\zeta^2 + \zeta^2)}_{=0} \\
&=
    \frac{1}{2}
    \tr \{ \bs H_{\alpha}^2 \}
    +
    \frac{1}{\sigma_0^2} \tr
    \{
    (\bs I-\bs H_{\alpha})
    [(\bs X\bs \beta_0)^{\otimes 2} + \sigma_0^2 \bs I]
    (\bs I-\bs H_{\alpha}) \bs H_{\alpha}
    \} \\
&=
    \tr\{
        \bs H_{\alpha}^2/2 
        +
        (\bs I-\bs H_{\alpha})^2
        \bs H_{\alpha}
    \}
    +
    \frac{1}{\sigma_0^2}\tr\{
        (\bs I-\bs H_{\alpha})\bs H_{\alpha}(\bs I-\bs H_{\alpha})
        (\bs X\bs \beta_0)^{\otimes 2}
    \} \\
&=
    \tr\{\bs H_{\alpha}\}
    -
    \frac{3}{2}\tr\{\bs H_{\alpha}^2\}
    +
    \tr\{\bs H_{\alpha}^3\}
    +
    \frac{1}{\sigma_0^2}\tr\{
        \bs H_{\alpha}((\bs I-\bs H_{\alpha})
        (\bs X\bs \beta_0))^{\otimes 2}
    \},
\end{align*}
which proves the assertion.
\qed

\section{Convergence of the Ornstein-Uhlenbeck Process}
\label{sec:supporting_propositions}

The following Proposition~\ref{prop:W1_evaluation} evaluates the $1$-Wasserstein distance between the distributions of the Ornstein-Uhlenbeck and the Langevin processes. 

\begin{prop}[Theorem~2 of \citet{cheng2018sharp} ]
\label{prop:W1_evaluation}
Let $R>0$ and let let $0 < \phi  \le \frac{p}{\sqrt{p/2\lambda_{\min}(\bs Q_{\alpha}^{-1})+1}}$ be a desired accuracy. 
Assume that $\exp(-R^2\lambda_{\max}(\bs Q_{\alpha}^{-1})) \le 2R^2 \lambda_{\min}(\bs Q_{\alpha}^{-1})$ and $\|\bs \gamma^{(0)}-\hat{\bs \beta}_{\alpha}\|_2 \le R$. 
Then, it holds for $\delta:=\frac{\phi^2\exp(-2 R^2 \lambda_{\max}(\bs Q_{\alpha}^{-1}))}{2^{10}R^2 p}$ and $t \asymp \exp\left( \frac{3R^2\lambda_{\max}(\bs Q_{\alpha}^{-1}) p}{\phi^2} \right)$ that 
\begin{align*}
    W_1(\mathcal{D}(\tilde{\bs \gamma}_{t\delta}),\mathcal{D}(\bs \gamma^{(t)})) 
    \le \phi,
\end{align*}
where $W_1(\cdot,\cdot)$ denotes $1$-Wasserstein distance.
\end{prop}

\begin{proof}
Substituting $m:=2\lambda_{\min}(\bs Q_{\alpha}^{-1})$ and $L:=2\lambda_{\max}(\bs Q_{\alpha}^{-1})$ to \citet{cheng2018sharp} Theorem~2 proves the assertion. 
\end{proof}
\end{appendices}

\end{document}